\documentclass{article}

\usepackage[preprint]{icml2026}

\usepackage{microtype}
\usepackage{graphicx}
\usepackage{adjustbox}
\usepackage{subcaption}  
\usepackage{booktabs}
\usepackage{threeparttable} 
\usepackage{placeins}

\usepackage{amsmath,amssymb,amsfonts,amsthm,bm,mathtools}
\usepackage{siunitx}

\usepackage{enumitem}
\usepackage{float} 

\graphicspath{{images/}}

\newcommand{\State}{\STATE}
\newcommand{\Require}{\REQUIRE}
\newcommand{\Ensure}{\ENSURE}

\newcommand{\For}{\FOR}
\newcommand{\EndFor}{\ENDFOR}

\newcommand{\vdelta}{\boldsymbol{\delta}}

\newcommand{\Return}{\textbf{return} }

\newcommand{\Comment}[1]{\hfill\(\triangleright\)\,#1}
\newcommand{\best}[1]{\ifmmode\mathbf{#1}\else\textbf{#1}\fi}

\newcommand{\R}{\mathbb{R}}
\newcommand{\C}{\mathbb{C}}
\newcommand{\E}{\mathbb{E}}
\newcommand{\Var}{\operatorname{Var}}
\newcommand{\Cov}{\operatorname{Cov}}
\newcommand{\KL}{\operatorname{KL}}
\newcommand{\iu}{\mathrm{i}}

\newcommand{\IFFT}{\mathcal{F}^{-1}}
\newcommand{\diag}{\operatorname{diag}}
\newcommand{\trace}{\operatorname{Tr}}
\newcommand{\vx}{\mathbf{x}}
\newcommand{\bW}{\mathbf{W}}
\newcommand{\vw}{\mathbf{w}}
\newcommand{\vy}{\mathbf{y}}
\newcommand{\vz}{\mathbf{z}}
\newcommand{\va}{\mathbf{a}}
\newcommand{\vg}{\mathbf{g}}
\newcommand{\bF}{\mathbf{F}}
\newcommand{\bK}{\mathbf{K}}

\newcommand{\bX}{\mathbf{X}}
\newcommand{\bY}{\mathbf{Y}}
\newcommand{\bI}{\mathbf{I}}
\newcommand{\conj}[1]{#1^{\ast}}
\newcommand{\vb}{\mathbf{b}}
\newcommand{\bU}{\mathbf{U}}
\newcommand{\bh}{\mathbf{h}}
\newcommand{\btheta}{\bm{\theta}}
\DeclareMathOperator{\RFFT}{RFFT}
\DeclareMathOperator{\IRFFT}{IRFFT}

\theoremstyle{plain}
\newtheorem{theorem}{Theorem}
\newtheorem{definition}[theorem]{Definition}
\newtheorem{corollary}[theorem]{Corollary}
\theoremstyle{definition}
\newtheorem{remark}[theorem]{Remark}
\newtheorem{proposition}[theorem]{Proposition}

\icmltitlerunning{Compact Spectral Circulant Layers}

\begin{document}

\twocolumn[
\icmltitle{Compact Circulant Layers with Spectral Priors}

\begin{icmlauthorlist}
\icmlauthor{Joseph Margaryan}{diku}
\icmlauthor{Thomas Hamelryck}{diku,bio}
\end{icmlauthorlist}

\icmlaffiliation{diku}{Department of Computer Science (DIKU), University of Copenhagen, Copenhagen, Denmark}
\icmlaffiliation{bio}{Department of Biology, University of Copenhagen, Copenhagen, Denmark}


\icmlcorrespondingauthor{Joseph Margaryan}{josephmargaryan@gmail.com}

\icmlcorrespondingauthor{Thomas Hamelryck}{thamelry@bio.ku.dk}

\icmlkeywords{Machine Learning, Spectral Methods, Bayesian Deep Learning}

\vskip 0.3in
]
\printAffiliationsAndNotice{}
\begin{abstract}
Critical applications in areas such as medicine, robotics and autonomous systems require compact (i.e., memory efficient), uncertainty-aware neural networks suitable for edge and other resource-constrained deployments. We study compact spectral circulant and block-circulant-with-circulant-blocks (BCCB) layers: FFT-diagonalizable circular convolutions whose weights live directly in the real FFT (RFFT) half (1D) or half-plane (2D). Parameterizing filters in the frequency domain lets us impose simple spectral structure, perform structured variational inference in a low-dimensional weight space, and calculate exact layer spectral norms, enabling inexpensive global Lipschitz bounds and margin-based robustness diagnostics. By placing independent complex Gaussians on the Hermitian support we obtain a discrete instance of the spectral representation of stationary kernels, inducing an exact stationary Gaussian-process prior over filters on the discrete circle/torus. We exploit this to define a practical spectral prior and a Hermitian-aware low-rank-plus-diagonal variational posterior in real coordinates. Empirically, spectral circulant/BCCB layers are effective compact building blocks in both (variational) Bayesian and point estimate regimes: compact Bayesian neural networks on MNIST$\rightarrow$Fashion-MNIST, variational heads on frozen CIFAR-10 features, and deterministic ViT projections on CIFAR-10/Tiny ImageNet, spectral layers match strong baselines while using substantially fewer parameters and with tighter Lipschitz certificates.
\end{abstract}

\section{Introduction}
\label{sec:intro}

Standard convolutional and dense networks can be overconfident and brittle under distribution shift, while Gaussian processes offer clean Bayesian semantics but scale poorly \citep{guo2017calibration,hendrycks2017baseline,rasmussen2006gaussian}. We aim for layers that are simultaneously parameter-efficient, uncertainty-aware, and with controllable spectral norms \citep{miyato2018spectral,bartlett2017spectrally}.

We investigate \emph{spectral circulant/BCCB layers}: FFT-diagonalizable circular convolutions whose trainable parameters are the nonredundant RFFT coefficients (1D half-spectrum / 2D half-plane). This exposes simple spectral masks/decays for compression and enables a diagonal proper complex-Gaussian prior on the Hermitian support, inducing an exact stationary GP prior over discrete filters with closed-form KL terms \citep{bochner1933monotone,rudin1962fourier}. Since these operators are exactly diagonalized by the DFT, their spectral norms follow directly from the frequency response, enabling essentially free Lipschitz diagnostics and margin-based certificates \citep{tsuzuku2018lipschitz}.

\paragraph{Contributions}
We show that by parameterizing circulant-structured neural layers directly in real-FFT coordinates, we can simultaneously leverage classical spectral GP theory, exact operator geometry, and practical low-rank variational inference in a modern deep learning setting.  

\begin{itemize}[leftmargin=*,itemsep=2pt,topsep=2pt]
\item \textbf{Spectral circulant/BCCB parameterization.}
We parameterize 1D circulant and 2D BCCB layers directly by their \emph{nonredundant} real-FFT (RFFT) coefficients. For 1D inputs $\vx\in\mathbb{R}^d$ (output $\vy\in\mathbb{R}^d$), the layer is implemented as
\[
\vy \;=\; \IRFFT\!\big(\bh_{\mathrm{half}} \odot \RFFT(\vx)\big),
\]
where $\bh_{\mathrm{half}}\in\C^{k_{\mathrm{half}}}$ denotes the nonredundant 1D RFFT half-spectrum (the free Fourier coefficients; the remaining coefficients are determined by Hermitian symmetry).
We use $\odot$ for elementwise (Hadamard) multiplication and $\conj{(\cdot)}$ for complex conjugation.
This corresponds to RFFT $\rightarrow$ frequencywise complex multiplication $\rightarrow$ inverse RFFT.
In 2D with multi-channel mixing, let $\bX\in\R^{C_{\mathrm{in}}\times H\times W}$ be the input (with $C_{\mathrm{in}}$ channels) and $\bY\in\R^{C_{\mathrm{out}}\times H\times W}$ the output.
A hat denotes the frequency-domain representation produced by the real 2D FFT, i.e.\ $\widehat{\bX}=\RFFT_2(\bX)$.
We learn the nonredundant half-plane frequency response
$\bK_{\mathrm{half}}\in\C^{C_{\mathrm{out}}\times C_{\mathrm{in}}\times H\times W_{\mathrm{half}}}$, where $W_{\mathrm{half}}=\lfloor W/2\rfloor+1$ and $\bK_{\mathrm{half}}[o,c]\in\C^{H\times W_{\mathrm{half}}}$ couples input channel $c$ to output channel $o$.
The forward pass is
\begin{align}
\widehat{\bX} &= \RFFT_2(\bX), \\
\widehat{\bY}[o] &= \sum_{c=1}^{C_{\mathrm{in}}}\bK_{\mathrm{half}}[o,c]\odot \widehat{\bX}[c], \\
\bY &= \IRFFT_2(\widehat{\bY}), \label{eq:contrib-bccb-forward}
\end{align}

Storing $\bh_{\mathrm{half}}$ / $\bK_{\mathrm{half}}$ (instead of spatial kernels) keeps only the independent Fourier coefficients (the remainder are fixed by Hermitian symmetry), and it avoids computing an FFT of the kernel during training. Band-limits and spectral decays are then implemented by simple masks or frequency-dependent scalings in this parameter space.

\item \textbf{Discrete spectral GP prior and Hermitian-aware SVI.}
We show that independent \emph{proper} (i.e., circularly symmetric; or, equivalently the real part $\Re z$ and imaginary part $\Im z$ are independent with equal variance and $\mathbb{E}[z^2]=0$) complex Gaussian coefficients on the Hermitian support, with a nonnegative variance profile, induce an exact stationary GP prior over discrete filters in 1D and 2D.
We use this to define a diagonal spectral prior and a Hermitian-aware low-rank-plus-diagonal Gaussian variational posterior in \emph{effective real coordinates}, i.e., the vector of free real degrees of freedom obtained by unpacking the nonredundant RFFT coefficients into real/imaginary parts while respecting Hermitian constraints.
This yields closed-form KL terms while capturing dominant cross-frequency/channel correlations.

\item \textbf{Compactness, uncertainty, and Lipschitz geometry.}
In small SVI BNNs on MNIST$\rightarrow$Fashion-MNIST, spectral circulant/BCCB layers deliver competitive ID accuracy, good calibration, and improved ID/OOD entropy separation with 7--72$\times$ fewer parameters than dense baselines. As SVI heads over frozen CIFAR-10 features, spectral circulant heads match dense heads' accuracy with $\approx 3\times$ fewer parameters and substantially smaller Lipschitz bounds and larger certified feature-space radii. Finally, in a small ViT on CIFAR-10, replacing square linear projections by spectral circulant layers improves accuracy and NLL while reducing the number of parameters in the projection layers.
\end{itemize}

\begin{table}[t]
  \centering
  \small
  \caption{What is new compared to prior work \emph{and why it matters}. Directly parameterizing the nonredundant RFFT coefficients aligns the prior/posterior with the trainable parameters in code (enabling simple, closed-form variational terms) and makes exact spectral norms and Lipschitz diagnostics essentially free for circulant/BCCB layers.}
  \label{tab:whats-new}
  \begin{tabular}{p{0.27\linewidth}p{0.67\linewidth}}
    \toprule
    Aspect & This work \\
    \midrule
    Parameterization &
      Circulant/BCCB layers parameterized directly in the \emph{nonredundant} RFFT half-spectrum/half-plane.
      As a consequence, we keep only independent Fourier coefficients (the remainder are fixed by Hermitian symmetry), avoid per-step kernel FFTs, and implement band-limits/decays via simple masks or frequency-dependent scalings. \\[2pt]
    Prior \& inference &
      A diagonal proper complex-Gaussian prior on these spectral coordinates.
      As a consequence, it induces an exact stationary GP prior over discrete filters, and it pairs naturally with a Hermitian-aware low-rank\,+\,diagonal Gaussian variational posterior in effective real coordinates with closed-form KL terms. \\[2pt]
    Geometry \& use cases &
      For FFT-diagonalizable circulant/BCCB operators, layer spectral norms are available in closed form.
      As a consequence, we can compute inexpensive global Lipschitz \emph{upper bounds} and margin-based $\ell_2$ certificates as geometric diagnostics (not robustness guarantees), and we evaluate these layers as compact replacements for square projections in ViT-style architectures. \\
    \bottomrule
  \end{tabular}
\end{table}

We refer to Table \ref{tab:notation} in the Appendix for an overview of the notation and symbols used. 

\section{Related Work}
\label{sec:related}

\paragraph{Structured linear operators and circulant layers.}
Circulant and block-circulant matrices are classical FFT-diagonalizable operators in signal processing \citep{davis1979circulant,gray2006toeplitz,oppenheim2010dtsp}. In deep learning, circulant and block-circulant structure has been used to compress fully-connected and convolutional layers while retaining accuracy \citep{7410684,lavin2016fast,rippel2015spectral}. FFT-domain convolution has also been explored to accelerate convnet training and inference \citep{Mathieu2013FastTO}. More broadly, low-displacement-rank structured matrices subsume Toeplitz/circulant-like families and come with universal-approximation-style guarantees \citep{Zhao2017TheoreticalPF}, and deep diagonal--circulant products have been analyzed and trained effectively \citep{Araujo2019UnderstandingAT,feng2025cdflow,Ding2025ParameterEfficientFW}. These works primarily target compression, fast training, or expressivity, and do not place an explicit GP-compatible prior directly on the nonredundant RFFT parameters optimized in code. They also do not exploit exact FFT diagonalization to obtain closed-form layer spectral norms for inexpensive Lipschitz diagnostics. We instead parameterize directly in the RFFT half/half-plane and endow Fourier coefficients with a GP-compatible spectral prior.

\paragraph{Spectral kernels and Gaussian processes.}
Bochner's theorem characterizes stationary kernels as Fourier transforms of nonnegative spectral measures \citep{bochner1933monotone,rasmussen2006gaussian}. This underpins spectral GP models, random Fourier features, and scalable approximations \citep{rahimi2007random,lazaro2010sparse,wilson2013spectral,wilson2015kissgp}. These constructions are usually phrased at the \emph{kernel} level. Our discrete spectral prior is instead implemented as a weight-space prior directly on the nonredundant RFFT coefficients used by standard real-FFT APIs, so the prior/posterior live on the exact parameters optimized in code. This makes the method a drop-in replacement for FFT-based convolutions (no per-iteration kernel FFTs) and keeps both inference and Hermitian bookkeeping local and transparent: we operate directly on the stored nonredundant RFFT coefficients and enforce Hermitian symmetry via a fixed reconstruction map, rather than relying on implicit constraints or repeated spatial$\rightarrow$frequency transforms of the kernel.

\paragraph{Bayesian neural networks and structured variational inference.}
Bayesian neural networks with Gaussian priors and variational approximations are well studied \citep{neal1996bayesian,hoffman2013stochastic,kingma2014autoencoding,rezende2014stochastic}. Structured posteriors (e.g., matrix Gaussians, low-rank-plus-diagonal) trade flexibility for tractability \citep{louizos2016structured,zhang2018noisy}. Our Hermitian-aware low-rank variational posterior lives on the effective real coordinates of the nonredundant RFFT coefficients and is tailored to the FFT-diagonal structure of circulant/BCCB layers.
As a result, we can model meaningful posterior correlations with modest rank while retaining cheap reparameterized sampling and closed-form KL terms, making structured SVI practical for these layers.

\section{Background: Circulant/BCCB Structure and the FFT}
\label{sec:background-fft-main}

A vector $\vw = (w_0,\ldots,w_{d-1})^\top \in \R^d$ determines a
\emph{circulant} linear map $\mathrm{circ}(\vw) \in \R^{d\times d}$ via
circular shifts (indices modulo $d$):
\[
\mathrm{circ}(\vw)
=
\begin{bmatrix}
w_0 & w_{d-1} & w_{d-2} & \cdots & w_1\\
w_1 & w_0     & w_{d-1} & \cdots & w_2\\
w_2 & w_1     & w_0     & \cdots & w_3\\
\vdots & \vdots & \vdots & \ddots & \vdots\\
w_{d-1} & w_{d-2} & w_{d-3} & \cdots & w_0
\end{bmatrix}.
\]

Equivalently, for $\vx\in\R^d$,
\[
\big(\mathrm{circ}(\vw)\,\vx\big)_t
=
\sum_{s=0}^{d-1} w_{(t-s)\bmod d}\,x_s,
\qquad t=0,\ldots,d-1,
\]
i.e., the map corresponds to a circular convolution.

Let $\bF\in\C^{d\times d}$ be the unitary DFT matrix with
$\bF_{k,t}=\frac{1}{\sqrt{d}}e^{-2\pi\iu kt/d}$, and let
$\widehat{\vw}=\sqrt{d}\,\bF\vw$.
Circulant matrices are exactly diagonal in the Fourier basis
\citep{davis1979circulant,gray2006toeplitz}:
\begin{equation}
\mathrm{circ}(\vw)=\bF^\dagger\diag(\widehat{\vw})\,\bF,
\label{eq:circ-diag-bg}
\end{equation}
so for any $\vx\in\R^d$ we have
\begin{equation}
\mathrm{circ}(\vw)\,\vx
=
\bF^\dagger\big(\widehat{\vw}\odot(\bF\vx)\big),
\label{eq:circ-matvec-bg}
\end{equation}
with $O(d\log d)$ cost using the FFT
\citep{cooley1965fft,oppenheim2010dtsp}.

For real $\vx$ and $\vw$, the spectrum is Hermitian,
$\widehat{w}_{d-k}=\conj{\widehat{w}_k}
$, so we only need the
nonredundant half of size
$k_{\mathrm{half}}=\lfloor d/2\rfloor+1$.
Real FFTs (RFFTs) implement this by returning the half-spectrum.
We treat these nonredundant RFFT coefficients as the primitive
parameters of our 1D spectral-circulant layers.

On a 2D grid with circular boundary conditions, convolutions correspond
to block-circulant-with-circulant-blocks (BCCB) operators that are
diagonalized by the separable 2D DFT \citep{gray2006toeplitz}.
For multi-channel inputs
$\bX\in\R^{C_{\mathrm{in}}\times H\times W}$ and spectral kernels
$\widehat{\bK}(u,v)\in\C^{C_{\mathrm{out}}\times C_{\mathrm{in}}}$,
the frequency-domain representation is
\[
\widehat{\bY}(u,v) = \widehat{\bK}(u,v)\,\widehat{\bX}(u,v),
\]
and the spatial output is recovered by an inverse 2D FFT:
\[
\bY = \IRFFT_2\big(\widehat{\bY}\big),
\]
where $\RFFT_2$ and $\IRFFT_2$ denote the real 2D FFT and its inverse.
While the RFFT half-plane storage itself follows standard real-FFT practice, our key design innovation is to elevate these nonredundant spectral coefficients to the \emph{learnable} parameters and to place both the GP-compatible prior and the Hermitian-aware variational posterior directly in this coordinate system.

\section{Spectral Circulant and BCCB Layers}
\label{sec:method}

We now describe the spectral parameterization of circulant and BCCB layers used throughout the paper. Full derivations and matrix-level expressions are given in App.~\ref{app:fft-background}.

\subsection{1D spectral-circulant layers}

A vector $\vw\in\R^d$ defines a circulant linear map
$\mathrm{circ}(\vw)\in\R^{d\times d}$ via circular convolution.
Let $\bF\in\C^{d\times d}$ be the unitary DFT and
$\widehat{\vw}=\sqrt{d}\,\bF\vw$.
Circulant matrices are exactly diagonal in the Fourier basis
\citep{davis1979circulant,gray2006toeplitz}:
\begin{equation}
\label{eq:circ-diag}
\mathrm{circ}(\vw)=\bF^\dagger \diag(\widehat{\vw})\,\bF.
\end{equation}
For any $\vx\in\R^d$ we therefore have
\begin{equation}
\label{eq:fft-forward-1d}
\mathrm{circ}(\vw)\,\vx
=
\bF^\dagger\big(\widehat{\vw}\odot(\bF\vx)\big),
\end{equation}

For real $\vx$ and $\vw$, the spectrum is Hermitian:
$\widehat{w}_{d-k}=\conj{\widehat{w}_k}$.
Therefore, it suffices to store the nonredundant \emph{half} of size
$k_{\mathrm{half}}=\lfloor d/2\rfloor+1$.
We denote these parameters by $\bh_{\mathrm{half}}\in\C^{k_{\mathrm{half}}}$.
The forward pass then becomes
\begin{equation}
\label{eq:rfft-forward-1d}
\vy=\IRFFT\big(\bh_{\mathrm{half}}\odot \RFFT(\vx)\big),
\end{equation}

where $\RFFT/\IRFFT$ are real FFTs on the half-spectrum.
The usual implementation hides Hermitian bookkeeping inside the FFT
library; our perspective is to treat $\bh_{\mathrm{half}}$ as the
\emph{primitive} parameters of the layer.

This representation has two immediate benefits:
(i) redundant conjugate coefficients are removed, and
(ii) the coordinates $\bh_{\mathrm{half}}$ are exactly those used by our spectral GP prior and variational posterior in Sec.~\ref{sec:gp-vi-main}.

\subsection{2D spectral BCCB layers with channel mixing}

On an $H\times W$ grid with circular boundary conditions, 2D convolutions correspond to BCCB operators that are diagonalized by the separable 2D DFT \citep{gray2006toeplitz}.
For inputs $\bX\in\R^{C_{\mathrm{in}}\times H\times W}$ we parameterize a
complex spectral kernel
\[
\bK_{\mathrm{half}}\in\C^{C_{\mathrm{out}}\times C_{\mathrm{in}}\times H\times W_{\mathrm{half}}}
\]
where $W_{\mathrm{half}}=\lfloor W/2\rfloor+1$, on the RFFT$_2$ half-plane.
Let $\widehat{\bX}=\RFFT_2(\bX)$.
The forward pass is
\begin{align}
\widehat{\bY}[o] &= \sum_{c=1}^{C_{\mathrm{in}}}
  \bK_{\mathrm{half}}[o,c]\odot \widehat{\bX}[c], \nonumber\\
\bY &= \IRFFT_2(\widehat{\bY}), \label{eq:bccb-forward}
\end{align}

with an optional per-output-channel bias added in the spatial domain.

The per-image complexity is
\[
(C_{\mathrm{in}}+C_{\mathrm{out}})HW\log(HW)
\;+\;
C_{\mathrm{in}}C_{\mathrm{out}}HW,
\]
for input/output FFTs and frequency-wise channel mixing, respectively. This is the same complexity as a standard FFT-based convolution, but we avoid per-step kernel FFTs because $\bK_{\mathrm{half}}$ is already stored in the frequency domain.

\label{app:rfft2d-forward}

\begin{algorithm}[t]
  \caption{RFFTCirculant2D forward pass}
  \label{alg:rfft2d-forward}
  \begin{algorithmic}[1]
    \Require $\bX \in \mathbb{R}^{C_{\mathrm{in}}\times H\times W}$ \Comment{input tensor}
    \Require $\bK_{\text{half}} \in \mathbb{C}^{C_{\mathrm{out}}\times C_{\mathrm{in}}\times H\times W_{\mathrm{half}}}$ \Comment{spectral kernel}
    \Require $\vb \in \mathbb{R}^{C_{\mathrm{out}}}$ \Comment{bias}
    \Ensure $\bY \in \mathbb{R}^{C_{\mathrm{out}}\times H\times W}$ \Comment{output tensor}
    \State $\widehat{\bX} \gets \RFFT_2(\bX)$
    \For{$o = 1$ to $C_{\mathrm{out}}$}
      \State $\widehat{\bY}[o] \gets 0$
      \For{$c = 1$ to $C_{\mathrm{in}}$}
        \State $\widehat{\bY}[o] \gets \widehat{\bY}[o] + \bK_{\text{half}}[o,c] \odot \widehat{\bX}[c]$
      \EndFor
    \EndFor
    \State $\bY \gets \IRFFT_2(\widehat{\bY})$
    \State $\bY[o] \gets \bY[o] + \vb[o]$ for all $o$
    \State \Return $\bY$
  \end{algorithmic}
\end{algorithm}

\subsection{Band-limiting}

The spectral parameterization exposes explicit hyperparameters for compactness.

\textbf{Band-limiting.}
In 1D we may activate only the lowest $K\le k_{\mathrm{half}}$ frequencies, zeroing higher ones.
In 2D we apply a radial mask with normalized cutoff $K_{\mathrm{rad}}\in[0,1]$ in frequency coordinates.
These choices directly control the effective capacity of the layer without leaving the RFFT representation.

\textbf{Prior spectrum (Bayesian only).}
When using the spectral GP prior, we parameterize its diagonal variance profile $S(\cdot)$ with a simple low-pass envelope controlled by a slope parameter $\alpha$; see Sec.~\ref{sec:gp-vi-main}.

\subsection{Backpropagation and complexity}
\label{sec:backprop-structure}

Because FFTs are linear (and unitary under our normalization), gradients through spectral circulant/BCCB layers have the same structure as the forward pass \citep{oppenheim2010dtsp}.
For a 1D spectral-circulant layer, gradients with respect to the weights and inputs can be written as circular correlations, implemented via FFTs with elementwise products and conjugation; see App.~\ref{app:fft-background} for explicit formulas.
Consequently, both forward and backward passes scale as $O(d\log d)$ in 1D (and $O(HW\log(HW))$ per channel in 2D), rather than $O(d^2)$ for an unstructured dense $d\times d$ map.

Overall, the dominant cost of both forward and backward passes is in the FFTs and frequencywise channel mixing.
The number of \emph{trainable} parameters, however, is controlled by the number of RFFT coefficients we choose to activate, which is typically much smaller than the number of entries in the corresponding dense matrix or spatial kernel.
This separation---FFT cost vs.\ parameter count---is central to the compression and compactness results in Sec.~\ref{sec:experiments-main}.

\section{Spectral Prior and Hermitian-aware SVI}
\label{sec:gp-vi-main}

We now describe the spectral prior and the variational posterior on the RFFT coefficients. Full proofs and implementation details are deferred to App.~\ref{app:gp-proofs}--\ref{app:vi-details}.

\subsection{Implementation notes}
\label{app:vi-implementation-notes}

Our implementation mirrors the factorization above by combining one custom variational distribution per spectral layer (handling Hermitian reconstruction in effective real coordinates) with a generic mean-field variational distribution for all remaining parameters.
We keep these software-specific details separate from the mathematical description because they do not affect the model definition or the ELBO, but they are useful for reproducibility.
\paragraph{Software.}
We implement Bayesian models and stochastic variational inference using NumPyro \citep{Phan:2019elc} on top of JAX \citep{jax2018github}. The deterministic Vision Transformer experiments are implemented in JAX \citep{jax2018github} using Equinox \citep{kidger2021equinox}. Anonymized code is provided in the supplementary material.

\subsection{Discrete spectral GP prior on filters}

\paragraph{Choice of prior spectrum (spectral slope).}
In experiments we use a simple low-pass envelope for the diagonal spectral prior variance profile:
\[
S(k)=\frac{\sigma_0^2}{1+\rho(k)^\alpha},\qquad \alpha\ge 0,
\]
where $\rho(\cdot)\in[0,1]$ is a normalized (wrapped) frequency radius.
In 1D on $\mathbb{Z}_n$, $\rho(k)=\min\{k,n-k\}/\max\{1,\lfloor n/2\rfloor\}$.
In 2D on $\mathbb{Z}_H\times\mathbb{Z}_W$, $\rho(u,v)=\sqrt{\rho_H(u)^2+\rho_W(v)^2}$ with
$\rho_H(u)=\min\{u,H-u\}/\max\{1,\lfloor H/2\rfloor\}$ and similarly for $\rho_W(v)$.
$\alpha=0$ gives a flat spectrum; larger $\alpha$ increasingly suppresses high frequencies.

Consider the discrete circle $\mathbb{Z}_n$. Let $S:\{0,\ldots,n{-}1\}\to[0,\infty)$ be an \emph{even} nonnegative spectrum with $S(n{-}k)=S(k)$. We sample a Hermitian-symmetric spectrum $F\in\C^n$ by drawing independent \emph{proper} (i.e., circularly symmetric) complex Gaussians on the non-self-conjugate frequencies and real Gaussians on self-conjugate (DC and Nyquist) bins with variance profile $S$, then set
\begin{equation}
\label{eq:ifft-prior}
w_t = \frac{1}{\sqrt{n}}\sum_{k=0}^{n-1}F_k\, e^{2\pi\iu kt/n},\quad t=0,\ldots,n{-}1.
\end{equation}

\begin{theorem}[Discrete spectral GP prior]
\label{thm:gp-main}
Under the construction above, $\vw=(w_0,\ldots,w_{n-1})^\top$ is jointly Gaussian, mean-zero, with stationary covariance
\begin{equation}
\label{eq:gp-cov}
k(\tau)=\Cov(w_t,w_{t+\tau})
=
\frac{1}{n}\sum_{k=0}^{n-1} S(k)e^{2\pi\iu k\tau/n},
\end{equation}
which depends only on the difference $\tau$ modulo $n$.
\end{theorem}

Thus any nonnegative even spectrum $S$ induces a stationary GP prior over discrete filters, and sampling independent complex-Gaussian coefficients with variance $S$ in the RFFT half (plus Hermitian completion) implements this prior. A 2D analogue on $\mathbb{Z}_H\times\mathbb{Z}_W$ yields stationary random fields for BCCB filters; see App.~\ref{app:gp-proofs}.

\paragraph{2D extension.}
On a discrete torus $\mathbb{Z}_H\times\mathbb{Z}_W$ we define an even nonnegative spectrum $S(u,v)\ge 0$ with $S(-u,-v)=S(u,v)$ modulo $(H,W)$.
Sampling a Hermitian-symmetric field $F_{u,v}$ with $\E[|F_{u,v}|^2]=S(u,v)$ and
$w_{x,y}=\frac{1}{\sqrt{HW}}\sum_{u,v}F_{u,v}e^{2\pi\iu(ux/H+vy/W)}$ yields a mean-zero Gaussian random field with stationary covariance
\[
\kappa(\tau_x,\tau_y)=\frac{1}{HW}\sum_{u,v}S(u,v)e^{2\pi\iu(u\tau_x/H+v\tau_y/W)}.
\]
This is the natural GP prior over 2D filters for our spectral BCCB layers; details are given in App.~\ref{app:gp-proofs}.

\subsection{Hermitian-aware low-rank variational posterior}

The RFFT half/half-plane stores only the nonredundant complex coefficients. Self-conjugate frequencies are real-valued; non-self-conjugate frequencies contribute real and imaginary parts paired by conjugacy. 
We collect all \emph{free} real degrees of freedom into an effective coordinate vector
\[
\va\in\R^{d_{\mathrm{eff}}}
\]
and define a fixed linear map $T:\R^{d_{\mathrm{eff}}}\to\C^{\text{half}}$ that reconstructs
the complex RFFT coefficients (including self-conjugate bins) from $\va$, and then
completes the full spectrum by Hermitian symmetry before the inverse FFT.

On $\va$ we place a structured Gaussian posterior
\begin{equation}
q(\va)=\mathcal{MVN}\!\Big(\bm\mu,\;\bU\diag(\bm\lambda^2)\bU^\top+\diag(\bm\sigma^2)+\varepsilon \bI\Big).
\end{equation}
Samples are obtained via the reparameterization trick:
\[
\va = \bm\mu + \bU(\bm\lambda\odot\xi) + \bm\sigma\odot\zeta,
\]
with $\xi\sim\mathcal{MVN}(0,\bI_r),\;\zeta\sim\mathcal{MVN}(0,\bI_{d_{\mathrm{eff}}})$,
followed by $\bh_{\mathrm{half}}=T(\va)$ and $\vw=\IRFFT(\bh_{\mathrm{half}})$.
The low-rank term captures global correlations across frequencies and channels; the diagonal part provides per-coordinate flexibility.

Because the spectral GP prior factorizes over frequencies in the complex domain, the induced prior on $a$ is a diagonal Gaussian with known variance entries, obtained by viewing each complex coefficient as a 2D real normal with variance $S/2$ per component and each real coefficient as a 1D normal with variance $S$. The KL divergence
\[
\KL\big(q(\va)\,\|\,p(\va)\big)
\]
therefore reduces to a closed-form KL between a low-rank-plus-diagonal covariance and a diagonal covariance, computed efficiently via the matrix determinant lemma and the Woodbury identity (details in App.~\ref{app:vi-details}). In practice we use a small rank (e.g.\ $r=8$) for all spectral layers.
\begin{algorithm}[t]
  \caption{Hermitian-aware low-rank guide for one spectral layer (SVI)}
  \label{alg:spec-guide}
  \begin{algorithmic}[1]
    \Require variational parameters $(\bm\mu,\bU,\bm\lambda,\bm\sigma)$, jitter $\varepsilon$, prior variances $\bm\tau^2$ on $\va$
    \Require reconstruction map $T:\R^{d_{\mathrm{eff}}}\to\C^{\text{half}}$, optional mask $M$ on the half-spectrum
    \State Sample $\xi\sim\mathcal{MVN}(0,\bI_r)$,\; $\zeta\sim\mathcal{MVN}(0,\bI_{d_{\mathrm{eff}}})$
    \State $\va \gets \bm\mu + \bU(\bm\lambda\odot\xi) + \bm\sigma\odot\zeta$
    \State $\bh_{\text{half}} \gets T(\va)$ \Comment{fills DC/Nyquist as real, packs conjugate pairs}
    \State $\bh_{\text{half}} \gets M \odot \bh_{\text{half}}$ \Comment{optional band-limit / radial mask}
    \State $\vw \gets \IRFFT(\bh_{\text{half}})$ \Comment{$\mathrm{IRFFT}$ (1D) or $\mathrm{IRFFT}_2$ (2D)}
    \State $\Sigma \gets \bU\diag(\bm\lambda^2)\bU^\top + \diag(\bm\sigma^2) + \varepsilon \bI$
    \State $\mathrm{KL} \gets \KL\!\big(\mathcal{MVN}(\bm\mu,\Sigma)\,\|\,\mathcal{MVN}(0,\diag(\bm\tau^2))\big)$
    \Comment{closed form}
    \State \Return sampled weights $\vw$ and prior term $\mathrm{KL}$
  \end{algorithmic}
\end{algorithm}

Algorithm~\ref{alg:spec-guide} summarizes the exact steps used in our implementation: sampling in effective real coordinates, Hermitian-safe reconstruction in RFFT storage, and an analytic Gaussian KL against the diagonal spectral prior.

\section{Lipschitz Bounds and Certificates}
\label{sec:lipschitz-main}

Lipschitz constants and spectral norms are widely used to characterize and control the sensitivity of neural networks to input perturbations, with connections to stability, margin-based generalization bounds, and certified robustness \citep{cisse2017parseval,miyato2018spectral,bartlett2017spectrally,gouk2021regularisation,tsuzuku2018lipschitz}. Here we leverage the fact that circulant and BCCB operators are \emph{exactly} diagonalized by the DFT, so their layer spectral norms can be computed directly from their frequency response (rather than estimated by iterative power methods), making Lipschitz diagnostics essentially free in our parameterization.

The FFT-diagonal structure lets us obtain exact layer spectral norms and simple global Lipschitz bounds.

\paragraph{1D spectral-circulant layers.}
For a 1D spectral-circulant layer,
\[
T_{\btheta}(\vx) \;=\; \IRFFT\!\big(\bh_{\mathrm{half}}(\btheta)\odot\RFFT(\vx)\big),
\]
let $\bh_{\mathrm{full}}(\btheta)\in\C^d$ denote the Hermitian-completed full spectrum corresponding to $\bh_{\mathrm{half}}(\btheta)$, and let $h_k(\btheta)$ denote its $k$th (scalar) entry. Then
\[
T_{\btheta} \;=\; \bF^\dagger \diag\!\big(\bh_{\mathrm{full}}(\btheta)\big)\,\bF.
\]
Since $T_{\btheta}$ is normal with eigenvalues $\{h_k(\btheta)\}_{k=0}^{d-1}$, its operator norm is
\begin{equation}
\label{eq:lipschitz-1d}
\|T_{\btheta}\|_{2\to 2} = \max_k |h_k(\btheta)|,
\end{equation}
which can be evaluated directly from the stored RFFT coefficients (the maximum over the half-spectrum suffices since conjugate pairs share magnitude).

\paragraph{A prior typical-case bound (``bad geometry is unlikely'')}
The equality in Eq.~\eqref{eq:lipschitz-1d} characterizes the \emph{worst-case} Lipschitz behavior of a spectral-circulant layer once its spectrum is fixed.
Under our diagonal spectral prior, however, the nonredundant Fourier coefficients are random and (by construction) independent across active frequency bins, which yields a simple high-probability control on unusually large spectral norms.

\begin{proposition}[Prior high-probability bound for spectral-circulant norms]
\label{prop:prior-lip}
Consider a 1D spectral-circulant layer with an active set of nonredundant RFFT bins $\mathcal{K}$ (after any band-limit mask), with $m=|\mathcal{K}|$.
Under the diagonal spectral prior, the active scalar Fourier coefficients $\{h_k\}_{k\in\mathcal{K}}$ are independent, with non-self-conjugate bins satisfying
$(\Re h_k,\Im h_k)\sim \mathcal{MVN}\!\big(0,\tfrac{S(k)}{2}\bI_2\big)$ (equivalently $h_k\sim\mathcal{CN}(0,S(k))$),
and self-conjugate bins (DC and, if applicable, Nyquist) satisfying $h_k\sim\mathcal{N}(0,S(k))$.

Let $S_{\max}=\max_{k\in\mathcal{K}} S(k)$. Then for any $t>0$,
\[
\Pr\!\Big[\|T\|_{2\to 2} \ge t\Big]
\;\le\;
2m\,\exp\!\Big(-\frac{t^2}{2S_{\max}}\Big),
\]
and consequently for any $\delta\in(0,1)$,
\[
\Pr\! \left[\|T\|_{2\to 2} \le \sqrt{2S_{\max}\log\!\frac{2m}{\delta}}\right]
\;\ge\; 1-\delta.
\]
\end{proposition}

In words, under the spectral prior the layer Lipschitz constant typically scales like $\sqrt{S_{\max}\log m}$; extremely large operator norms are exponentially unlikely.

\paragraph{2D multi-channel BCCB layers.}
For 2D BCCB layers with multi-channel mixing, the 2D DFT diagonalizes the spatial structure but leaves channel mixing as small matrices. Flattening spatial dimensions, the layer is unitarily equivalent to a block-diagonal map
\[
\begin{aligned}
& \widehat{y}_k = K_{\btheta}(k)\widehat{x}_k, \\
& D_{\btheta} = \mathrm{blockdiag}(K_{\btheta}(k_1),\ldots,K_{\btheta}(k_{HW})),
\end{aligned}
\]
where $K_{\btheta}(k)\in\C^{C_{\mathrm{out}}\times C_{\mathrm{in}}}$ are frequencywise mixing matrices. The operator norm is therefore \citep{sedghi2019singular}
\begin{equation}
\label{eq:lipschitz-2d}
\|T_{\btheta}\|_{2\to 2} = \max_k \sigma_{\max}(K_{\btheta}(k)),
\end{equation}
where $\sigma_{\max}$ denotes the largest singular value. Hermitian symmetry does not change singular values, so the maximum can be taken over the RFFT$_2$ half-plane.

\paragraph{Network-level bounds and certificates.}
For a network $f_{\btheta}$ that is a composition of linear layers and 1-Lipschitz activations (e.g.\ $\tanh$), ignoring biases a standard product bound yields
\begin{align}
\operatorname{Lip}(f_{\btheta})
&\le \widehat{\operatorname{Lip}}(f_{\btheta}), \\
\widehat{\operatorname{Lip}}(f_{\btheta})
&:=
\prod_{\ell\in\mathcal{L}_{\mathrm{spec}}}
\|T_{\ell,\btheta}\|_{2\to 2}
\;\cdot\;
\prod_{\ell\in\mathcal{L}_{\mathrm{dense}}}
\|\mathbf{W}_{\ell}(\btheta)\|_2, \label{eq:netlip}
\end{align}
where for 1D spectral-circulant layers $\|T_{\ell,\btheta}\|_{2\to 2}=\max_k |h_{\ell,k}(\btheta)|$ (Eq.~\eqref{eq:lipschitz-1d})
and for 2D BCCB layers with channel mixing $\|T_{\ell,\btheta}\|_{2\to 2}=\max_k \sigma_{\max}(K_{\ell,\btheta}(k))$ (Eq.~\eqref{eq:lipschitz-2d}).
Under independent spectral priors, Proposition~\ref{prop:prior-lip} yields a simple simultaneous control of these factors; see Corollary~\ref{cor:prior-net-lip} in App.~\ref{app:lip-details}.

Given logits $f_{\btheta}(\vx)\in\R^K$ and true label $y$, the logit margin
\[
m_{\btheta}(\vx,y)=f_{\btheta}(\vx)_y-\max_{k\neq y}f_{\btheta}(\vx)_k
\]
yields a standard global $\ell_2$ robustness certificate
\begin{equation}
\label{eq:cert-radius}
r_{\mathrm{cert}}(x,y;\btheta)
=
\frac{\max\{m_{\btheta}(x,y),0\}}{2\,\widehat{\operatorname{Lip}}(f_{\btheta})}.
\end{equation}
In our Bayesian experiments we sample $\btheta\sim q(\btheta)$, compute $m_{\btheta}$ and $r_{\mathrm{cert}}$ on test inputs, and study their empirical posterior distributions (see App.~\ref{app:lip-details}).

\section{Experiments}
\label{sec:experiments-main}

We evaluate spectral circulant/BCCB layers in three regimes:
(1) small \emph{variational Bayesian} (SVI) models trained from scratch on MNIST (ID) with Fashion-MNIST (OOD),
(2) \emph{variational} (SVI) probabilistic heads on a frozen CIFAR-10 encoder,
and (3) a deterministic Vision Transformer (ViT) on CIFAR-10.

Unless noted otherwise, results are averaged over three seeds.
Full training details, OOD protocol, and additional plots are given in App.~\ref{app:exp-details}.

\subsection{General setup}

For Bayesian models we use reparameterized SVI with Adam, a fixed learning rate and budget shared within each dataset group, and the Hermitian-aware low-rank guide from Sec.~\ref{sec:gp-vi-main} with rank $r=8$ for all spectral layers.
We report in-distribution (ID) accuracy, negative log-likelihood (NLL), and calibration metrics, and use predictive entropy on OOD data to compute AUROC and FPR@95\%TPR.
Formal metric definitions and implementation details follow standard practice \citep{guo2017calibration,hendrycks2017baseline} and are summarized in App.~\ref{app:ood-protocol}.

\subsection{MNIST $\rightarrow$ Fashion-MNIST}
\label{sec:exp-mnist}

We train small SVI BNNs from scratch on MNIST \citep{lecun1998gradient} and evaluate OOD performance on Fashion-MNIST \citep{xiao2017fashionmnist}.
All models share the form
\[
\text{layer} \;\to\; \tanh \;\to\; \text{linear} \;\to\; \text{softmax},
\]
with the hidden layer chosen among:
\emph{Spectral BCCB} (ours), \emph{Spectral Circulant} (ours, 1D on flattened input), Conv2D, and a Dense $D\to D$ layer with $D=28\cdot 28$.

Spectral models are extremely compact: on $28\times 28$ images, spectral BCCB 
and spectral circulant variants use about $8.6$k weights versus $62.8$k for 
Conv2D and $622$k for the dense baseline; see parameter counts in 
App.~\ref{app:param-counts} (Table~\ref{tab:param_counts_mnist}).

Table~\ref{tab:mnist-main-full} summarizes ID accuracy and OOD performance using 
predictive entropy; full metrics are in App.~\ref{app:mnist-details}.

\begin{table}[t]
\centering
\small
\caption{MNIST (ID) and Fashion-MNIST (OOD) with SVI BNNs. Params are trainable weights (excluding biases). AUROC is based on predictive entropy; full metrics are in App.~\ref{app:mnist-details}.}
\label{tab:mnist-main-full}
\setlength{\tabcolsep}{3.5pt}
\begin{tabular}{l r c c c c}
\toprule
Model & Params & Acc  & Brier & ECE & AUROC \\
\midrule
Spectral BCCB      & 8.6k  & 0.92 & 0.12 & 0.02 & 0.81 \\
Spectral Circ.     & 8.6k  & 0.92 & 0.12 & 0.02 & 0.83 \\
Conv2D             & 62.8k & 0.96 & 0.08 & 0.04 & 0.62 \\
Dense              & 622k  & 0.97 & 0.05 & 0.02 & 0.84 \\
\bottomrule
\end{tabular}
\end{table}

Spectral models achieve strong calibration (low Brier/ECE) and clear ID/OOD entropy separation while using 7--72$\times$ fewer parameters than the dense baseline.
Compared with Conv2D under the same SVI budget, spectral BCCB and spectral circulant layers achieve slightly lower accuracy but substantially higher OOD AUROC, indicating more cautious predictions on Fashion-MNIST.

Figure~\ref{fig:mnist-entropy} shows kernel-density estimates (KDEs) of predictive entropy for MNIST (ID) versus Fashion-MNIST (OOD) under the same SVI budget.
In each row, the left panel shows the full entropy range and the right panel shows an identical zoom into the low-entropy region to make ID/OOD separation at high confidence visible.

\begin{figure}[t]
  \centering

  \begin{minipage}[t]{0.49\columnwidth}
    \centering
    \scriptsize
    \makebox[\linewidth]{\textbf{Full range (left)}\hfill\textbf{Zoom (right)}}
    \vskip 0.2em
    \normalsize
    \includegraphics[width=\linewidth]{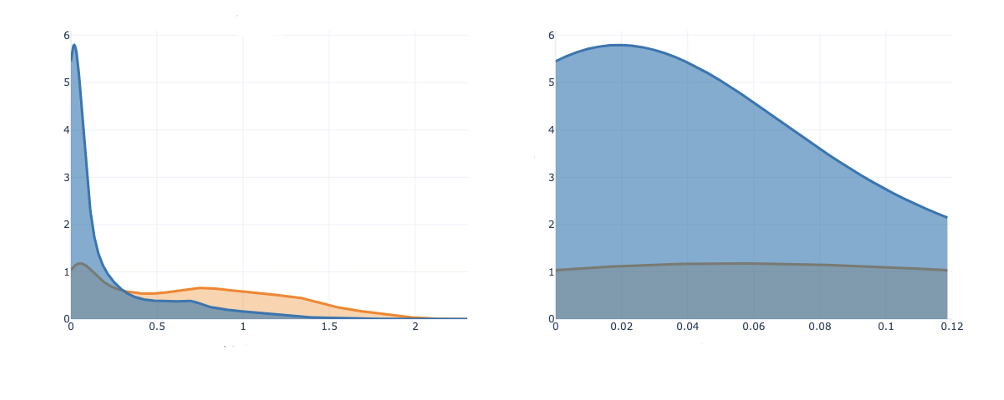}
  \end{minipage}\hfill
  \begin{minipage}[t]{0.49\columnwidth}
    \centering
    \scriptsize
    \makebox[\linewidth]{\textbf{Full range (left)}\hfill\textbf{Zoom (right)}}
    \vskip 0.2em
    \normalsize
    \includegraphics[width=\linewidth]{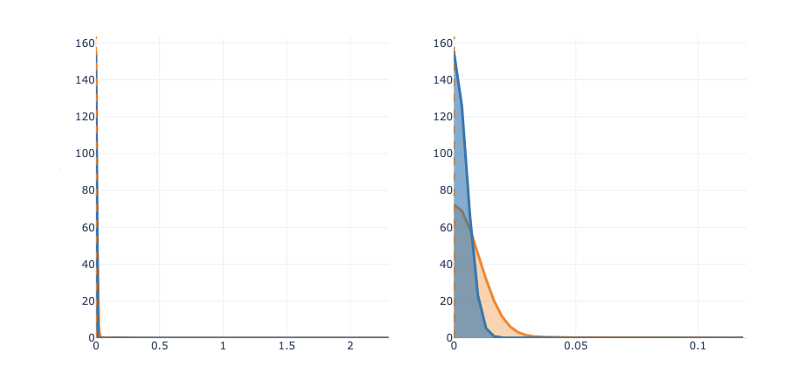}
  \end{minipage}

  \caption{Predictive-entropy KDEs for MNIST (ID; \textbf{blue}) versus Fashion-MNIST (OOD; \textbf{orange}), computed from the SVI posterior predictive. Left: Spectral BCCB (ours). Right: Conv2D baseline. In each image, the left panel shows the full entropy range and the right panel shows the same zoom window $H\in[0,0.12]$ nats.}
  \label{fig:mnist-entropy}
\end{figure}

\subsection{Frozen CIFAR-10 encoder $\rightarrow$ Bayesian heads}
\label{sec:exp-frozen}

We map frozen pooled features $\vz\in\R^{2048}$ to logits in $\R^{10}$.
We compare (i) MLP heads $2048\to h\to 10$ with $h\in\{32,128\}$ (Dense-32/128),
and (ii) a circulant projection $2048\to 2048$ (spatial or spectral) followed by a
linear classifier $2048\to 10$. Concretely, the heads are implemented as either
Dense-32 or Dense-128 (single hidden layer), a spatial circulant layer, or a
spectral circulant layer parameterized in the RFFT half as in
Sec.~\ref{sec:method}.

For the spectral head we vary the active half-spectrum size $K$; here we report the full half-spectrum $K=d/2+1$ and refer to App.~\ref{app:dk-ablation} for a detailed $K$-ablation.

Table~\ref{tab:cifar-heads-main} reports ID and OOD metrics on CIFAR-10 and CIFAR-10-C \citep{hendrycks2019benchmarking} respectively.

\begin{table}[t]
\centering
\small
\setlength{\tabcolsep}{3pt} 
\caption{Frozen CIFAR-10 features $\rightarrow$ SVI heads.
We report CIFAR-10 ID accuracy/NLL/ECE and CIFAR-10-C mean AUROC and FPR@95\%TPR over corruptions and severities.
Higher is better for Acc/AUROC; lower is better for NLL/ECE/FPR@95.
Best values in each column are in bold.}

\label{tab:cifar-heads-main}
\begin{tabular*}{\columnwidth}{@{\extracolsep{\fill}}lccccc}
\toprule
Model & Acc & NLL & ECE & AUROC & FPR@95 \\
\midrule
Dense-32              & 0.971 & \best{0.098} & \best{0.009} & 0.70 & 0.57 \\
Dense-128             & 0.971 & 0.118        & 0.012        & \best{0.71} & \best{0.56} \\
Spectral Circ. (full) & 0.971 & 0.136        & 0.017        & 0.68 & 0.60 \\
Circulant (spatial)   & 0.965 & 0.433        & 0.026        & 0.64 & 0.63 \\
\bottomrule
\end{tabular*}
\end{table}
In weight count, Dense-32 uses 65,856 weights in the head, while the full spectral head uses 22,528 (about $2.9\times$ fewer; see App.~\ref{app:dk-ablation} for counts and the $K$-ablation).

Spectral circulant heads match Dense-32 in ID accuracy and track dense heads closely in NLL and AUROC while using about $3\times$ fewer parameters (see App.~\ref{app:param-counts}).
As $K$ shrinks, performance degrades gracefully: AUROC and FPR@95\%TPR remain stable until the effective spectrum becomes very small, indicating that the learned spectral prior and Hermitian-aware posterior can make efficient use of a limited frequency budget.

Using the Lipschitz bounds from Sec.~\ref{sec:lipschitz-main}, we also inspect the geometry of the heads.
For Dense-32 and the full spectral circulant head, we estimate global $\ell_2$ Lipschitz bounds and certified feature-space radii $r_{\mathrm{cert}}(x,y;\btheta)$ in Eq.~\eqref{eq:cert-radius} under the variational posterior.
Spectral heads consistently exhibit much smaller Lipschitz bounds and larger certified radii than dense heads at comparable predictive performance; full histograms and training trajectories are shown in App.~\ref{app:lip-details}.

\subsection{Vision Transformer with spectral projections}
\label{sec:exp-vit}

Finally, we test spectral circulant layers as drop-in replacements for square projections in a small ViT \citep{dosovitskiy2020image}. We replace all $d_{\mathrm{model}}\!\to\!d_{\mathrm{model}}$ linear maps in attention/MLP blocks by spectral circulant layers (patch embedding and classifier head remain dense) and train with the same recipe as App.~\ref{app:training-setup}. Table~\ref{tab:vit-main} summarizes CIFAR-10 and Tiny ImageNet results.

\begin{table}[t]
\centering
\small
\caption{Dense vs spectral-circulant ViT. CIFAR-10 numbers are from the main experiment; Tiny ImageNet numbers are provided for additional context.}
\label{tab:vit-main}
\begin{tabular}{lccccc}
\toprule
Dataset & Model & Params & Acc & NLL \\
\midrule
CIFAR-10 & Dense ViT    & 1{,}205{,}898 & 0.685 & 0.97 \\
CIFAR-10 & Spectral ViT &   811{,}218   & 0.716 & 0.88 \\
TinyImg  & Dense ViT    & 1{,}248{,}840 & 0.307 & 3.06 \\
TinyImg  & Spectral ViT &   854{,}160   & 0.344 & 2.94 \\
\bottomrule
\end{tabular}
\end{table}

On both datasets the spectral variant uses roughly $1.4$--$1.5\times$ fewer parameters than the dense model and attains better likelihood and accuracy under identical hyperparameters.
Figure~\ref{fig:vit-training} illustrates training dynamics. 

\begin{figure}[t]
  \centering
  \includegraphics[width=\columnwidth]{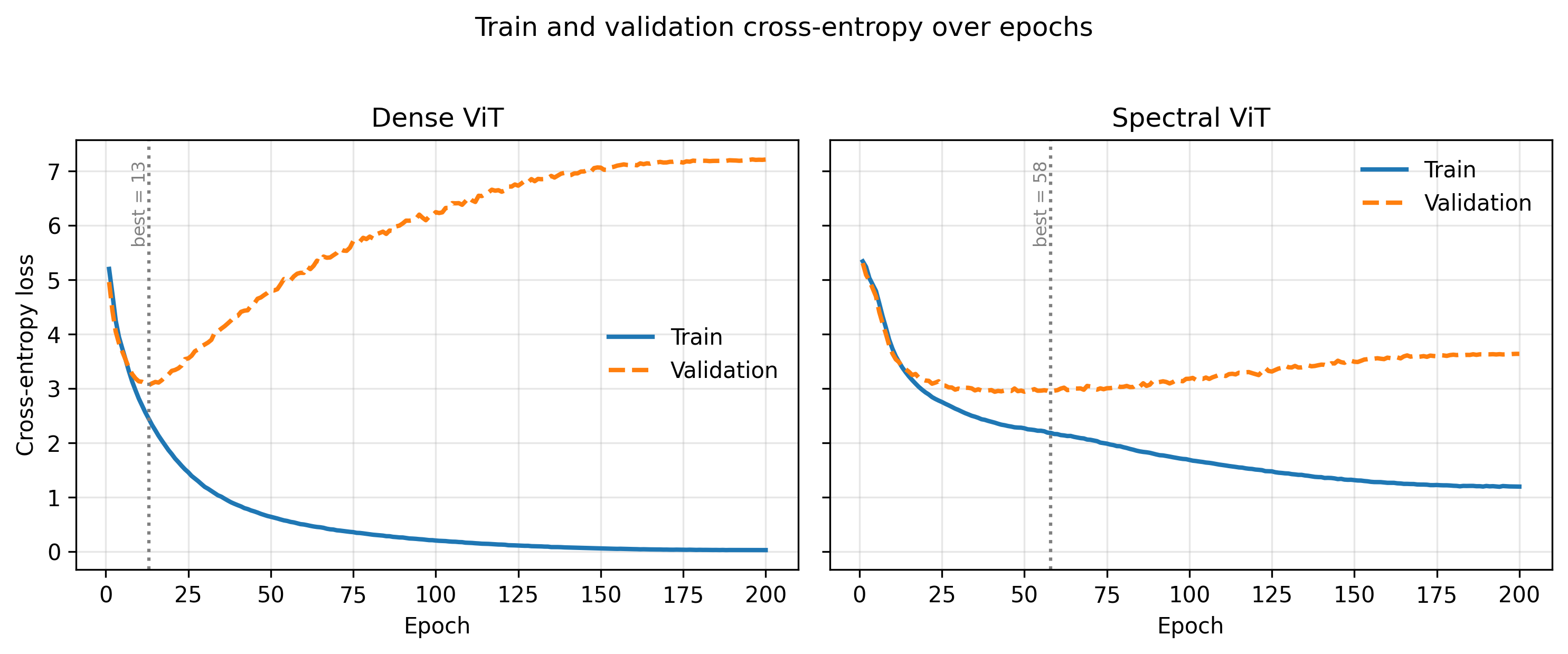}
  \caption{Train and validation cross-entropy over epochs for dense vs spectral ViT on Tiny ImageNet. Vertical lines mark best-validation checkpoints.}
  \label{fig:vit-training}
\end{figure}

On Tiny ImageNet, the spectral ViT reaches its best validation loss substantially later than the dense baseline (Fig.~\ref{fig:vit-training}), consistent with delayed overfitting. Under identical hyperparameters it achieves higher accuracy and lower NLL while using fewer parameters (Table~\ref{tab:vit-main}); additional accuracy--parameter and curvature diagnostics are reported in App.~\ref{app:vit-details} and App.~\ref{app:curvature-details}.

\paragraph{Appendix roadmap.}
 App.~\ref{app:fft-background} derives FFT/BCCB identities and gradients;
App.~\ref{app:gp-proofs} proves the discrete spectral GP result;
App.~\ref{app:vi-details} gives Hermitian bookkeeping and KL computation details;
App.~\ref{app:lip-details} contains Lipschitz/certificate derivations and extra plots;
and App.~\ref{app:exp-details} reports full experimental protocols, ablations, and additional figures.

\section{Discussion and Limitations}
\label{sec:discussion-main}

We parameterize circulant/BCCB layers directly by the nonredundant RFFT coefficients used in standard real-FFT APIs, aligning implementation with (i) a GP-compatible diagonal spectral prior, (ii) a Hermitian-aware structured variational posterior with closed-form KL, and (iii) exact layer spectral norms for cheap Lipschitz diagnostics. Empirically, these layers act as compact replacements for dense/conv projections in Bayesian heads and small transformers.

\paragraph{Limitations.}
Circular boundary conditions can introduce edge artifacts; padding or hybrid designs may mitigate this but break strict BCCB structure. Our spectral prior uses simple radial/1D profiles; learning richer spectra is a natural extension. At CIFAR scale we apply Bayesian inference only to the head; scaling structured SVI to deep spectral stacks remains challenging. Our Lipschitz analysis is global, $\ell_2$-based, and used only as a geometric diagnostic.

\clearpage
\section*{Impact Statement}

This work develops compact spectral layers and associated probabilistic priors for deep networks. The main anticipated impact is methodological: more parameter-efficient and better-calibrated models for standard supervised tasks, with potential downstream benefits in settings where memory or energy budgets are constrained. As with most work on general-purpose representation learning, the techniques here could be used in both socially beneficial and harmful applications. We do not target safety-critical domains directly and do not propose mechanisms that meaningfully change the risk profile of deploying machine learning systems beyond improving calibration and providing coarse global robustness diagnostics. Any use of these methods in high-stakes domains should be accompanied by domain-specific validation, careful monitoring, and alignment with relevant legal and ethical standards.

\bibliographystyle{icml2026}
\bibliography{references}

\clearpage

\appendix

\onecolumn

\section{Notation and Symbols}
\label{app:notation}

We follow the notation conventions used in standard deep learning references \citep{goodfellow2016deep}:
scalars are plain lowercase (e.g., $a,x$), vectors are bold lowercase (e.g., $\vx$), and
matrices/tensors are bold uppercase (e.g., $\bX,\bK$).

\begin{table*}[h]
  \centering
  \small
  \caption{Selected notation used in the main text and appendix.}
  \label{tab:notation}
  \begin{tabular}{ll}
    \toprule
    Symbol & Description \\
    \midrule
     $\R, \C$ & Real and complex numbers \\
    $n,m$ & Scalars (real or integer indices) \\
    $\vx \in \R^d$ & Column vector (e.g., input signal) \\
    $\vw \in \R^d$ & Filter vector (spatial domain) \\
     \midrule
    $\vy, \vz$ & Output or latent vectors \\
    $\bX \in \R^{C_{\mathrm{in}}\times H\times W}$ &
      Input tensor (channels $\times$ height $\times$ width) \\
    $\bY \in \R^{C_{\mathrm{out}}\times H\times W}$ &
      Output tensor \\
    $\bF \in \C^{d\times d}$ &
      Unitary DFT matrix (1D) with entries
      $\bF_{k,t} = \frac{1}{\sqrt{d}}e^{-2\pi\iu kt/d}$ \\
    $\bK(u,v) \in \C^{C_{\mathrm{out}}\times C_{\mathrm{in}}}$ &
      Frequencywise channel-mixing matrix at 2D frequency $(u,v)$ \\
    $\bI$ & Identity matrix \\
    \midrule
    $\mathrm{circ}(\vw) \in \R^{d\times d}$ &
      Circulant matrix generated by $\vw$ \\
    BCCB & Block-circulant-with-circulant-blocks operator (2D circular convolution) \\
    \midrule
    $\widehat{\vw} = \bF\vw$ &
      Discrete Fourier transform (DFT) of $\vw$ \\
    $k_{\mathrm{half}} = \lfloor d/2\rfloor+1$ &
      Size of the nonredundant 1D RFFT half-spectrum \\
    $\bh_{\mathrm{half}}$ &
          Complex RFFT coefficients for a 1D spectral-circulant layer
          (nonredundant half-spectrum) \\
        $\bh_{\mathrm{full}}$ &
          Hermitian-completed full spectrum vector corresponding to $\bh_{\mathrm{half}}$; $h_k$ denotes its $k$th (scalar) entry. \\
    $\bK_{\mathrm{half}}$ &
      Complex RFFT$_2$ coefficients (half-plane) for a 2D BCCB layer \\
    $\RFFT, \IRFFT$ &
      1D real FFT / inverse real FFT operators \\
    $\RFFT_2, \IRFFT_2$ &
      2D real FFT / inverse real FFT operators \\
    
    $\odot$ & Elementwise (Hadamard) product \\
    
    $\Re(\cdot),\,\Im(\cdot)$ & Real and imaginary parts of a complex number \\
    
    $\conj{c}$ & Complex conjugate of $c$ \\
    
    \midrule
    $S(k)$, $S(u,v)$ &
      Nonnegative spectral variance profile (prior) in 1D / 2D \\
    $\alpha$ &
      Spectral slope hyperparameter controlling low-/high-frequency bias \\
    \midrule
    $\E[\cdot]$ & Expectation operator \\
    $\Var(\cdot)$, $\Cov(\cdot,\cdot)$ &
      Variance and covariance \\
    $\KL(\cdot\|\cdot)$ &
      Kullback--Leibler divergence \\
    $\diag(\cdot)$ &
      Diagonal matrix with given entries \\
    $\trace(\cdot)$ &
      Trace of a square matrix \\
    \midrule
    $\mathcal{N}(\mu,\sigma^2)$ &
          Univariate normal on $\R$. \\
    $\mathcal{MVN}(\bm\mu,\Sigma)$ &
          Multivariate normal on $\R^d$. \\
    $\mathcal{CN}(0,S)$ &
          Proper complex Gaussian: $z\sim\mathcal{CN}(0,S)\iff(\Re z,\Im z)\sim\mathcal{MVN}\!\big(0,\tfrac{S}{2}\bI_2\big)$. \\
    $q(\cdot)$, $p(\cdot)$ &
      Variational posterior and prior distributions \\
    $\btheta$ &
          Generic model parameter vector. \\
    $\theta$ &
          Generic scalar parameter / hyperparameter. \\
    $\va \in \R^{d_{\mathrm{eff}}}$ &
      Effective real coordinates for nonredundant RFFT degrees of freedom \\
    \bottomrule
  \end{tabular}
\end{table*}

\FloatBarrier

\section{Background: Circulant/BCCB and the FFT}
\label{app:fft-background}

This appendix collects additional details on circulant/BCCB structure,
FFT-based implementations, and gradients. The basic circulant matrix
schema and its diagonalization are introduced in
Sec.~\ref{sec:background-fft-main}.

\paragraph{Diagonalization by the DFT.}
Let $\vw\in\R^d$ and let $\bF\in\C^{d\times d}$ be the unitary DFT
matrix with $\bF_{k,t}=\frac{1}{\sqrt{d}}e^{-2\pi\iu kt/d}$, and
$\bF^\dagger=\bF^{-1}$ its conjugate transpose.
Writing $\widehat{\vw}=\bF\vw$, circulant matrices are exactly diagonal
in the Fourier basis \citep{davis1979circulant,gray2006toeplitz}:
\begin{equation}
\label{eq:circ_diag_app}
\mathrm{circ}(\vw)
=
\bF^\dagger \,\diag(\widehat{\vw})\, \bF.
\end{equation}
Thus for any $\vx\in\R^d$,
\begin{equation}
\label{eq:fft_forward_app}
\mathrm{circ}(\vw)\,\vx
=
\bF^\dagger\!\big(\,(\bF\vw) \odot (\bF\vx)\,\big),
\end{equation}
where $\odot$ denotes elementwise (Hadamard) product. With the Fast
Fourier Transform (FFT) \citep{cooley1965fft,oppenheim2010dtsp},
\eqref{eq:fft_forward_app} costs $O(d\log d)$.
\paragraph{Real inputs and the real FFT.}
When vectors are real, the spectrum exhibits Hermitian symmetry, so we
can store only the nonredundant $k=0,\ldots,\lfloor d/2\rfloor$
coefficients, halving memory and flops. In that case,
\begin{equation}
\label{eq:rfft_forward_app}
\mathrm{circ}(\vw)\,\vx
=
\IRFFT\!\big(\,\RFFT(\vw) \odot \RFFT(\vx)\,\big),
\end{equation}
with $\RFFT/\IRFFT$ understood as real-valued FFTs on the half-spectrum.
\paragraph{Backpropagation through a circulant layer.}
Consider a (pre-activation) vector
$\va = \mathrm{circ}(\vw)\,\vx$ and a scalar loss $\mathcal{L}$.
Let $\vg=\partial\mathcal{L}/\partial \va$.
Since FFT/IFTT are linear and unitary (under our normalization), the
adjoint maps are again Fourier transforms.
The gradients can be written as circular \emph{correlations} and admit
the same $O(d\log d)$ structure:
\begin{align}
\label{eq:grad_w_app}
\frac{\partial \mathcal{L}}{\partial \vw}
&= \mathrm{corr}(\vg,\vx) \\
&=
\IRFFT\!\big(\,\conj{\RFFT(\vg)}\;\odot\;\RFFT(\vx)\,\big),
\\[4pt]
\label{eq:grad_x_app}
\frac{\partial \mathcal{L}}{\partial \vx}
&= \mathrm{corr}(\vg,\vw) \\
&=
\IRFFT\!\big(\,\conj{\RFFT(\vg)}\;\odot\;\RFFT(\vw)\,\big),
\end{align}
where the asterix denotes complex conjugation.%
\footnote{Convolution uses no conjugation in frequency
(cf.\ Eq.~\eqref{eq:rfft_forward_app}); correlation introduces
conjugation on the first argument, yielding
Eqs.~\eqref{eq:grad_w_app}--\eqref{eq:grad_x_app}. Library-dependent
FFT normalizations change only constant factors.}
Hence the basic computational unit for both forward and backward passes
is
\[
\IRFFT\!\big(\,\RFFT(\cdot)\;(\text{or }\overline{\RFFT(\cdot)})\;\odot\;\RFFT(\cdot)\,\big),
\]
in contrast to $O(d^2)$ dense matvecs in unstructured layers.

\paragraph{2D BCCB operators and multi-channel mixing.}
On an $H\times W$ grid with circular boundary conditions, 2D
convolutions correspond to block-circulant-with-circulant-blocks (BCCB)
operators, diagonalized by the separable 2D DFT
\citep{gray2006toeplitz}. 

For multi-channel inputs
$\bX\in\R^{C_{\mathrm{in}}\times H\times W}$ and spectral kernels
$\widehat{\bK}\in\C^{C_{\mathrm{out}}\times C_{\mathrm{in}}\times H\times (W/2+1)}$
on the RFFT$_2$ half-plane, we have
\[
\widehat{\bY}[o] = \sum_{c=1}^{C_{\mathrm{in}}}\widehat{\bK}[o,c]\odot \widehat{\bX}[c]
\]

and

\[
\bY = \IRFFT_2(\widehat{\bY}),
\]
with $\widehat{\bX}=\RFFT_2(\bX)$. The per-image complexity is
\[
\underbrace{(C_{\mathrm{in}}{+}C_{\mathrm{out}})\,HW\log(HW)}_{\text{I/O FFTs}}
\;+\;
\underbrace{C_{\mathrm{in}}C_{\mathrm{out}}\,HW}_{\text{frequencywise mixing}},
\]
and the same structure carries to gradients by linearity of the FFT.

\section{GP--Equivalence of the Spectral--Circulant Prior}
\label{app:gp-proofs}

This appendix formalizes a key probabilistic fact underlying our spectral prior: drawing independent (proper) complex Gaussian Fourier coefficients with an even nonnegative variance profile and enforcing Hermitian symmetry yields a real-valued stationary Gaussian process on the discrete circle (and analogously a stationary Gaussian random field on the discrete torus). This is a standard instance of the Bochner--Herglotz characterization of positive-definite functions on compact abelian groups; we include the statement and proof in our notation for completeness and to make the GP connection explicit \citep{rudin1962fourier,bochner1933monotone}.

\paragraph{Conventions.}
Indices are taken modulo $n$. We use the unitary inverse DFT
\[
(\IFFT\,\mathbf{f})_t \;=\; \frac{1}{\sqrt n}\sum_{k=0}^{n-1} f_k\,e^{2\pi i kt/n}.
\]

\begin{definition}[Spectral–circulant sampling]\label{def:spectral-sampling}
Let $S:\{0,\ldots,n{-}1\}\to[0,\infty)$ be such that $S(n{-}k)=S(k)$ (even on $\mathbb Z_n$).
Draw independent \emph{proper} complex Gaussians on the Hermitian half:
for $k=1,\ldots,\lfloor (n{-}1)/2\rfloor$ let
$a_k,b_k\stackrel{\text{iid}}{\sim}\mathcal N(0,S(k)/2)$ and set $f_k=a_k+i b_k$;
set $f_0\sim\mathcal N(0,S(0))$ and, if $n$ is even, $f_{n/2}\sim\mathcal N(0,S(n/2))$;
complete by conjugacy $f_{n-k}=\overline{f_k}$. Then $\mathbf w=(w_0,\ldots,w_{n-1})$ with
\[
w_t \;=\; (\IFFT\,\mathbf f)_t \;=\; \frac{1}{\sqrt n}\sum_{k=0}^{n-1} f_k\,e^{2\pi i kt/n}
\]
is real by construction.
\end{definition}

\begin{remark}[Second–order structure]\label{rem:second}
The samples satisfy $\E[f_k\,\overline{f_\ell}]=S(k)\,\delta_{k\ell}$ and
$\E[f_k f_\ell]=S(k)\,\delta_{\ell,n-k}$ for all $k,\ell$.
\end{remark}

\begin{theorem}[GP–equivalence on $\mathbb Z_n$]\label{thm:gp-eq}
$\mathbf w$ is jointly Gaussian, mean zero, with stationary covariance
\[
\begin{aligned}
k(\tau)
 & :=\Cov\!\big(w_t,w_{t+\tau}\big) \\
 & =\frac{1}{n}\sum_{k=0}^{n-1} S(k)\,e^{2\pi i k\tau/n}, \\ 
 & \qquad \tau=0,\ldots,n{-}1.
\end{aligned}
\]
\end{theorem}

\begin{proof}
Gaussianity: $\IFFT$ is real–linear in $(\Re f,\Im f)$, which are jointly Gaussian.
For the covariance, by zero mean and Remark~\ref{rem:second},
\[
\begin{aligned}
k(\tau)
&=\E[w_t w_{t+\tau}] \\
& =\frac{1}{n}\sum_{k,\ell} e^{2\pi i(kt+\ell(t+\tau))/n}\,\E[f_k f_\ell] \\
& =\frac{1}{n}\sum_{k=0}^{n-1} S(k)\,e^{2\pi i(kt+(n-k)(t+\tau))/n} \\
& =\frac{1}{n}\sum_{k=0}^{n-1} S(k)\,e^{2\pi i k\tau/n},
\end{aligned}
\]
since $e^{2\pi i t}=e^{2\pi i \tau}=1$ for integer $t,\tau$.
\end{proof}

\begin{corollary}[Discrete Bochner/Herglotz]\label{cor:bochner}
For any $\mathbf{c}\in\C^n$,
\[
\sum_{t,s}\overline{c_t}c_s\,k(t{-}s)
=\frac{1}{n}\sum_{k=0}^{n-1} S(k)\,\Big|\sum_{s=0}^{n-1} c_s e^{-2\pi i ks/n}\Big|^2\ge 0.
\]
Hence $k\succeq 0$ iff $S\ge 0$.
\end{corollary}

\paragraph{2D extension (BCCB).}
Let $S(u,v)\!\ge\!0$ on $\{0,\ldots,H{-}1\}\times\{0,\ldots,W{-}1\}$ with $S(-u,-v)=S(u,v)$ modulo $(H,W)$.
Sample a Hermitian–symmetric field $\{f_{u,v}\}$ with $\E[|f_{u,v}|^2]=S(u,v)$ and set
\[
w_{x,y}=\frac{1}{\sqrt{HW}}\sum_{u=0}^{H-1}\sum_{v=0}^{W-1}
f_{u,v}\,e^{2\pi i (ux/H+vy/W)}.
\]
Then $(w_{x,y})$ is mean–zero Gaussian with stationary covariance
\[
\kappa(\tau_x,\tau_y)
=\frac{1}{HW}\sum_{u=0}^{H-1}\sum_{v=0}^{W-1}
S(u,v)\,e^{2\pi i (u\tau_x/H+v\tau_y/W)}.
\]

\section{Variational Inference Details}
\label{app:vi-details}

We now describe the variational posterior used for spectral layers.
The design goals are:
(i) respect the support implied by Hermitian symmetry (real DC/Nyquist, conjugate pairs);
(ii) model correlated posterior uncertainty across frequencies and channels; and
(iii) retain closed-form KLs against the diagonal spectral prior from Sec.~\ref{sec:gp-vi-main}, so the ELBO stays simple.

\paragraph{Effective coordinates.}
For a given spectral layer, let $\mathbf{f}_{\mathrm{half}}$ denote the nonredundant complex
RFFT coefficients on the Hermitian half-spectrum (1D) or half-plane (2D) used by the implementation.
Only a subset of real and imaginary parts are free: self-conjugate frequencies
(DC and Nyquist, when present; and the corresponding boundary frequencies in 2D) are real-valued,
with imaginary parts deterministically zero. We collect all free real degrees of freedom into
an effective coordinate vector $\va \in \R^{d_{\mathrm{eff}}}$ and define a fixed linear map
\[
T:\R^{d_{\mathrm{eff}}}\to \C^{\text{half}},\qquad \mathbf{f}_{\mathrm{half}} = T(\va),
\]
which packs $\va$ into the RFFT storage format (inserting the real self-conjugate bins and forming
real/imaginary pairs for the remaining frequencies). The corresponding spatial filter is recovered by
\[
\vw = \IRFFT(\mathbf{f}_{\mathrm{half}})
\qquad
(\text{or } \bW = \IRFFT_2(\mathbf{f}_{\mathrm{half}}) \text{ in 2D).}
\]
All maps are fixed and linear.

\paragraph{Hermitian-aware low-rank Gaussian.}
On $\va$ we choose a low-rank\,+\,diagonal Gaussian:
\begin{equation}
\label{eq:lowrank-family}
q(\va)\;=\;\mathcal{MVN}\!\Big(\bm\mu,\;\bU\,\diag(\bm\lambda^2)\bU^\top\;+\;\diag(\bm\sigma^2)\;+\;\varepsilon \bI\Big).
\end{equation}
with rank $r$ (we use $r{=}8$), mean $\bm\mu\in\mathbb{R}^{d_{\mathrm{eff}}}$,
factor-loading matrix $\bU\in\mathbb{R}^{d_{\mathrm{eff}}\times r}$,
scale vector $\bm\lambda\in\mathbb{R}^r$, diagonal variance $\bm\sigma^2\in\mathbb{R}^{d_{\mathrm{eff}}}$,
and small jitter $\varepsilon>0$ for numerical stability.
Samples are obtained via a standard reparameterization:
\[
\va \;=\; \bm\mu \;+\; \bU(\bm\lambda\odot \xi)\;+\;\bm\sigma\odot \zeta,
\quad
\xi\sim\mathcal{MVN}(0,\bI_r),\;\zeta\sim\mathcal{MVN}(0,\bI_{d_{\mathrm{eff}}}),
\]
followed by $\mathbf{f}_{\mathrm{half}}=T(\va)$ and $\vw=\IRFFT(\mathbf{f}_{\mathrm{half}})$. (or $\bW=\IRFFT_2(\mathbf{f}_{\mathrm{half}})$ in 2D).
The low-rank term $\bU(\bm\lambda\odot\xi)$ captures global correlations across frequencies and channels; the diagonal term $\bm\sigma\odot\zeta$ provides per-coordinate flexibility.

\paragraph{Compatibility with the spectral GP prior.}
Recall that the prior in Sec.~\ref{sec:gp-vi-main} factorizes over frequencies in the complex domain.
Each non-self-conjugate frequency has a proper complex Gaussian prior with variance $S(k)$ (or $S(u,v)$ in 2D), while self-conjugate bins are real with variance $S$.
Because $T$ is linear, the induced prior on effective coordinates $\va$ is again Gaussian with a \emph{known} diagonal covariance $C_{\text{prior}}$.
For practical purposes, one can view each complex bin as a 2D real normal with covariance $\frac{S}{2}I_2$, and each real bin as 1D normal with variance $S$.
The KL term
\[
\KL\big(q(\va)\,\|\,p(\va)\big)
\]
in the ELBO is then a closed-form KL between Gaussians with a low-rank\,+\,diagonal covariance and a diagonal covariance.
We compute it using the matrix determinant lemma and the Woodbury identity.
This keeps the prior term analytic and cheap.

\paragraph{Spectral hyperparameters.}
When learning a spectral slope $\alpha$ (Sec.~\ref{sec:method}) we introduce an unconstrained scalar $\alpha_z$ and a monotone map $\alpha=g(\alpha_z)\in[0,\infty)$, e.g.\ a softplus.
We either:
(i) treat $\alpha_z$ as part of a base mean-field guide (with its own Gaussian factor), or
(ii) include it directly in the layer-specific guide.
In both cases the prior on $\alpha_z$ is Gaussian and contributes another closed-form Gaussian KL term.

\paragraph{Variational distribution factorization.}
In a multi-layer network the full variational posterior factorizes as
\[
q(\btheta)\;=\;\prod_{\ell \in \mathcal{L}_{\text{spec}}} q_\ell(\va_\ell)\;
\times\;
q_{\mathrm{base}}\big(\btheta\setminus \{ \text{spectral sites}\}\big),
\]
where each $q_\ell$ is a Hermitian-aware low-rank Gaussian over the spectral layer $\ell$,
and $q_{\mathrm{base}}$ is a standard mean-field Gaussian over all remaining parameters (e.g.\ non-spectral weights, biases, and possibly spectral hyperparameters).
This separation keeps the Hermitian bookkeeping local to each spectral layer while leaving the remaining parameters in a simple variational family.

\paragraph{SVI objective and update.}
Let $\mathcal{D}$ denote the data and $\btheta$ all parameters, including spectral and nonspectral ones.
Our ELBO is
\[
\mathcal{L}(\phi) =
\E_{q_\phi(\btheta)}\!\big[\log p(\mathcal D \mid \btheta)\big] 
- \sum_{\ell\in \mathcal{L}_{\text{spec}}}
  \KL\!\big(q_\phi(\va_\ell)\,\|\,p(\va_\ell)\big) 
- \KL\!\big(q_\phi(\btheta_{\mathrm{base}})\,\|\,p(\btheta_{\mathrm{base}})\big).
\]

where $\phi$ collects all variational parameters and $p(\va_\ell)$ is the spectral GP prior in effective coordinates.
We estimate the expectation with the reparameterization trick, and compute the Gaussian KL terms in closed form.

A single SVI step for one minibatch proceeds as:
\begin{enumerate}[leftmargin=*, itemsep=1pt]
\item For each spectral layer $\ell$:
sample $\va_\ell\sim q_\phi(\va_\ell)$,
compute $\mathbf{f}_{\ell,\mathrm{half}}=T(\va_\ell)$ and $\vw_\ell=\IRFFT(\mathbf{f}_{\ell,\mathrm{half}})$ (or $\IRFFT_2$ in 2D).
\item Sample nonspectral parameters from $q_\phi(\btheta_{\mathrm{base}})$.
\item Run a forward pass through the network with these sampled parameters to evaluate $\log p(\mathcal D\,|\,\btheta)$.
\item Accumulate the closed-form KLs and update $\phi$ with stochastic gradients of the ELBO.
\end{enumerate}

\paragraph{Complexity.}
For a spectral layer with $d_{\mathrm{eff}}$ effective coordinates and rank $r$, the parameter cost of~\eqref{eq:lowrank-family} is $O(d_{\mathrm{eff}}r + d_{\mathrm{eff}})$; sampling and KL evaluation are $O(d_{\mathrm{eff}}r)$.
These costs depend on the number of Fourier coefficients, not on the spatial grid, and are dominated by the FFTs in the forward/backward passes.
We found $r{=}8$ to be a robust default across all experiments.

\section{Lipschitz Bounds and Margin Certificates: Details}
\label{app:lip-details}

The FFT-diagonal parameterization of our layers gives exact spectral norms essentially ``for free''.
We now spell out the resulting global $\ell_2$ Lipschitz bounds and the margin-based certificates used in the experiments, and we provide the proof of the prior high-probability bound from the main text.

\paragraph{1D spectral-circulant layers.}
Consider a 1D spectral-circulant layer with half-spectrum
$\bh_{\text{half}}\in\C^{k_{\text{half}}}$, acting on $\vx\in\R^d$ as
\begin{equation}
T_{\btheta}\vx \;=\; \IRFFT\!\big(\bh_{\text{half}}(\btheta)\odot \RFFT(\vx)\big).
\end{equation}
Let $\bF$ denote the unitary DFT matrix on $\C^d$ and
$\bh_{\mathrm{full}}(\btheta)\in\C^d$ the Hermitian-completed full spectrum corresponding to $\bh_{\text{half}}(\btheta)$, with scalar entries $h_k(\btheta)$.
Then
\begin{equation}
T_{\btheta} \;=\; \bF^\dagger \diag\!\big(\bh_{\mathrm{full}}(\btheta)\big)\,\bF,
\end{equation}
which is normal, with eigenvalues given by the entries $\{h_k(\btheta)\}_{k=0}^{d-1}$ of $\bh_{\mathrm{full}}(\btheta)$.
Hence the operator norm equals the maximum eigenvalue magnitude:
\begin{equation}
\label{eq:lipschitz-1d-app}
\|T_{\btheta}\|_{2\to 2}
\;=\;
\max_{0\le k\le d-1} |h_k(\btheta)|
\;=\;
\max_{0\le k\le k_{\mathrm{half}}-1} \big|(\bh_{\mathrm{half}}(\btheta))_k\big|.
\end{equation}

\paragraph{Prior typical-case bound.}
Proposition~\ref{prop:prior-lip} (main text) shows that unusually large spectral norms are exponentially unlikely under the diagonal spectral prior. We prove it here.

\begin{proof}[Proof of Proposition~\ref{prop:prior-lip}]
Let $\mathcal{K}_{\mathrm{cplx}}$ denote the active non-self-conjugate bins and
$\mathcal{K}_{\mathrm{real}}$ the active self-conjugate bins (DC and, if applicable, Nyquist),
so $\mathcal{K}=\mathcal{K}_{\mathrm{cplx}}\cup\mathcal{K}_{\mathrm{real}}$ and $m=|\mathcal{K}|$.

For $k\in\mathcal{K}_{\mathrm{cplx}}$, the diagonal spectral prior draws
$h_k = A_k + iB_k$ with $A_k,B_k \stackrel{\text{iid}}{\sim}\mathcal{N}(0,S(k)/2)$.
Hence $|h_k|^2/S(k)\sim \mathrm{Exp}(1)$ and therefore
\begin{equation}
\Pr\!\big[|h_k|\ge t\big] \;=\; \exp\!\Big(-\frac{t^2}{S(k)}\Big).
\label{eq:complex-tail}
\end{equation}
For $k\in\mathcal{K}_{\mathrm{real}}$, we have $h_k\sim \mathcal{N}(0,S(k))$, so the standard Gaussian tail bound gives
\begin{equation}
\Pr\!\big[|h_k|\ge t\big] \;\le\; 2\exp\!\Big(-\frac{t^2}{2S(k)}\Big).
\label{eq:real-tail}
\end{equation}
Since $\exp(-t^2/S(k)) \le 2\exp(-t^2/(2S(k)))$ for all $t\ge 0$, both cases satisfy the uniform bound
\begin{equation}
\Pr\!\big[|h_k|\ge t\big] \;\le\; 2\exp\!\Big(-\frac{t^2}{2S(k)}\Big),
\qquad k\in\mathcal{K}.
\label{eq:uniform-tail}
\end{equation}

By a union bound over active bins,
\begin{align}
\Pr\!\Big[\max_{k\in\mathcal{K}} |h_k| \ge t\Big]
&\le \sum_{k\in\mathcal{K}} \Pr\!\big[|h_k|\ge t\big] \nonumber\\
&\le \sum_{k\in\mathcal{K}} 2\exp\!\Big(-\frac{t^2}{2S(k)}\Big) \nonumber\\
&\le 2m\,\exp\!\Big(-\frac{t^2}{2S_{\max}}\Big),
\label{eq:unionbound}
\end{align}
where $S_{\max}=\max_{k\in\mathcal{K}} S(k)$.
Finally, Eq.~\eqref{eq:lipschitz-1d-app} implies $\|T_{\btheta}\|_{2\to 2}=\max_{k\in\mathcal{K}} |h_k|$, yielding the stated tail bound.
Setting $t=\sqrt{2S_{\max}\log\!\frac{2m}{\delta}}$ gives $\Pr\!\big[\|T_{\btheta}\|_{2\to 2}\le t\big]\ge 1-\delta$.
\end{proof}

\begin{corollary}[Prior typical-case control of the product Lipschitz bound]
\label{cor:prior-net-lip}
Consider a network with $L$ spectral-circulant layers and 1-Lipschitz activations, and let $\widehat{\operatorname{Lip}}(f_{\btheta})$ denote the product upper bound defined above.
Assume the $L$ spectral layers have independent diagonal spectral priors, with active bin counts $m_\ell$ and maxima $S_{\ell,\max}$.
Then for any $\delta\in(0,1)$, with probability at least $1-\delta$ the following holds simultaneously for all spectral layers:
\[
\|T_\ell\|_{2\to 2} \le \sqrt{2S_{\ell,\max}\log\!\frac{2m_\ell L}{\delta}},
\qquad \ell=1,\ldots,L,
\]
and therefore
\[
\prod_{\ell=1}^L \|T_\ell\|_{2\to 2}
\;\le\;
\prod_{\ell=1}^L \sqrt{2S_{\ell,\max}\log\!\frac{2m_\ell L}{\delta}}.
\]
\end{corollary}

\paragraph{2D multi-channel BCCB layers.}
For 2D BCCB layers with multiple channels, the same principle holds frequencywise.
Let $\bX\in\R^{C_{\mathrm{in}}\times H\times W}$,
$\bY\in\R^{C_{\mathrm{out}}\times H\times W}$, and spectral kernel
$\bK_{\text{half}}
\in\C^{C_{\mathrm{out}}\times C_{\mathrm{in}}\times H\times (W/2+1)}$.
Using the full complex 2D DFT for clarity, flatten the input/output across spatial dimensions as
$\vx\in\R^{C_{\mathrm{in}}HW}$ and $\vy\in\R^{C_{\mathrm{out}}HW}$, and apply the 2D DFT channelwise via unitary maps
\begin{equation}
\mathcal{F}_{\mathrm{in}} := \bI_{C_{\mathrm{in}}}\otimes \bF_{2\mathrm{D}},
\qquad
\mathcal{F}_{\mathrm{out}} := \bI_{C_{\mathrm{out}}}\otimes \bF_{2\mathrm{D}},
\end{equation}
where $\bF_{2\mathrm{D}}$ is the unitary 2D DFT on the spatial grid.
In the frequency domain, the layer acts blockwise as
\begin{equation}
\widehat{\vy}_k = K_{\btheta}(k)\,\widehat{\vx}_k,
\qquad k\in\{0,\ldots,HW{-}1\},
\end{equation}
with channel-mixing matrices $K_{\btheta}(k)\in\C^{C_{\mathrm{out}}\times C_{\mathrm{in}}}$.
Stacking over $k$ yields a block-diagonal map
\begin{equation}
\widehat{\vy} = D_{\btheta} \widehat{\vx},
\qquad
D_{\btheta} = \mathrm{blockdiag}\big(K_{\btheta}(k_1),\ldots,K_{\btheta}(k_{HW})\big),
\end{equation}
and the spatial operator is
\begin{equation}
T_{\btheta} = \mathcal{F}_{\mathrm{out}}^\dagger D_{\btheta} \mathcal{F}_{\mathrm{in}}.
\end{equation}
Unitary pre-/post-multiplication preserves the operator norm, and the norm of a block-diagonal matrix is the maximum block norm, giving
\begin{align}
\label{eq:bccb-lip-main}
\|T_{\btheta}\|_{2\to 2}
&=
\|D_{\btheta}\|_{2\to 2}
=
\max_{k}\,\|K_{\btheta}(k)\|_{2\to 2}
=
\max_{(u,v)} \sigma_{\max}\!\big(K_{\btheta}(u,v)\big),
\end{align}
where $\sigma_{\max}$ denotes the largest singular value.
Hermitian symmetry constrains which frequencies are stored but does not change these singular values, so the maximum can be taken over the nonredundant RFFT$_2$ half-plane.
In practice, these per-frequency norms involve only small $C_{\mathrm{out}}\times C_{\mathrm{in}}$ matrices and can be computed exactly via SVD; a few power-iteration steps are a cheaper approximation when needed.

\paragraph{Network-level Lipschitz bounds.}
Consider a network $f_{\btheta}$ obtained by composing linear layers $\{T_{\ell,\btheta}\}$ and 1-Lipschitz activations (e.g.\ $\tanh$).
Ignoring biases (which are translations), a standard product bound yields
\begin{align}
\operatorname{Lip}(f_{\btheta})
&\le \widehat{\operatorname{Lip}}(f_{\btheta}) \\
\widehat{\operatorname{Lip}}(f_{\btheta})
&:=
\prod_{\ell \in \mathcal{L}_{\text{spec}}}
\biggl(
  \max_{(u,v)} \sigma_{\max}\bigl(K_{\ell,\btheta}(u,v)\bigr)
\biggr)
\cdot
\prod_{\ell \in \mathcal{L}_{\text{dense}}}
\|\bW_\ell(\btheta)\|_2 .
\end{align}
Here $\mathcal{L}_{\text{spec}}$ and $\mathcal{L}_{\text{dense}}$ denote the sets of spectral BCCB and dense layers, respectively.
For spectral circulant (1D) layers, the corresponding factor is $\|T_{\ell,\btheta}\|_{2\to 2}=\max_k |h_{\ell,k}(\btheta)|$ (Eq.~\eqref{eq:lipschitz-1d-app} applied to layer $\ell$).

\paragraph{Margin-based $\ell_2$ certificates.}
Given logits $f_{\btheta}(\vx)\in\R^K$ and a true label $y$, define the logit margin
\begin{equation}
m_{\btheta}(\vx,y)
=
f_{\btheta}(\vx)_y - \max_{k\neq y} f_{\btheta}(\vx)_k.
\end{equation}
A standard Lipschitz--margin argument then yields a global $\ell_2$ robustness certificate \citep{tsuzuku2018lipschitz,hein2017formal}
\begin{equation}
r_{\mathrm{cert}}(\vx,y;\btheta)
\;=\;
\frac{\max\{m_{\btheta}(\vx,y),0\}}{2\,\widehat{\operatorname{Lip}}(f_{\btheta})}.
\end{equation}
For any perturbation $\vdelta$ with $\|\vdelta\|_2 \le r_{\mathrm{cert}}(\vx,y;\btheta)$, the predicted label cannot change.
In our Bayesian experiments we sample $\btheta\sim q(\btheta)$, compute $m_{\btheta}(\vx,y)$ and $r_{\mathrm{cert}}(\vx,y;\btheta)$ on the test set, and study their empirical posterior distributions (Sec.~\ref{sec:exp-frozen}).
These certificates are global and $\ell_2$-based; they are intended as a diagnostic of the geometry induced by spectral heads rather than a full adversarial-robustness claim.

\section{Experimental Details}
\label{app:exp-details}

\subsection{OOD Evaluation Protocol}
\label{app:ood-protocol}

We evaluate out-of-distribution (OOD) detection using predictive probabilities
$p_{\btheta}(\cdot\mid \vx)\in\Delta^{K-1}$ over $K$ classes. All logarithms are natural (nats).

\paragraph{Predictive entropy.}
\[
H\!\big(p_{\btheta}(\cdot\mid \vx)\big)
\;=\;
-\sum_{k=1}^{K} p_{\btheta}(y{=}k\mid \vx)\,\log p_{\btheta}(y{=}k\mid \vx).
\]

\paragraph{ID metrics.}
Negative log-likelihood (NLL), Brier, and calibration follow
\citet{guo2017calibration}:
\[
\begin{aligned}
&\text{NLL} \;=\; -\frac{1}{N}\sum_{i=1}^{N}\log p_{\btheta}(y_i\mid \vx_i) \\
& \text{Brier} \;=\; \frac{1}{N}\sum_{i=1}^{N}\big\|p_{\btheta}(\cdot\mid \vx_i)-\mathbf{e}_{y_i}\big\|_2^2
\end{aligned}
\]
\[
\begin{aligned}
& \text{ECE} \;=\; \sum_{m=1}^{M}\frac{|B_m|}{N}\,\big|\text{acc}(B_m)-\text{conf}(B_m)\big| \\
& \text{MCE}\;=\;\max_m \big|\text{acc}(B_m)-\text{conf}(B_m)\big|.
\end{aligned}
\]
Here $B_m$ are $M{=}15$ equal-width confidence bins in $[0,1]$, $\text{conf}(B_m)$ is
the average $\max_k p_{\btheta}(y{=}k\mid \vx)$ in the bin, and $\text{acc}(B_m)$ is the
fraction correct in that bin.

\paragraph{OOD scores and thresholds.}
Following standard practice \citep{hendrycks2017baseline} and the survey in
\citet{lu2024recentood}, we use \emph{negative entropy} as an ID score:
$s(\vx) = -H\big(p_{\btheta}(\cdot\mid \vx)\big)$ (larger $\Rightarrow$ more ID-like).
\[
\text{AUROC} \;=\; \text{AUC}\big(\{(s(\vx_i^{\text{ID}}),1)\}\cup\{(s(\vx_j^{\text{OOD}}),0)\}\big),
\]
\[
\begin{aligned}
& \text{FPR@95\%TPR} \;=\;  \Pr_{\vx^{\text{OOD}}}\!\big[s(\vx^{\text{OOD}})\ge t_{0.95}\big] \\
&\text{where } t_{0.95} 
\text{ is the }5\text{th percentile of } s(\vx^{\text{ID}}).
\end{aligned}
\]

\paragraph{Entropy KDE plots.}
We display kernel density estimates of $H$ on $[0,\ln K]$, using boundary
reflection and Scott’s bandwidth; figures report both full range and a zoomed
window (default $[0,0.12]$). 

\subsection{Training Setup and Hyperparameters}
\label{app:training-setup}

\textbf{Datasets.} From scratch: MNIST (ID) and Fashion-MNIST (OOD). Transfer: CIFAR-10 (ID) and CIFAR-10-C (OOD) with severities 1--5, using a frozen encoder. For the deterministic transformer experiment we train a small Vision Transformer (ViT) end-to-end on CIFAR-10.

\textbf{Backbones.} From scratch (images): a single spectral or spatial layer $\to$ \texttt{tanh} $\to$ linear head. We compare:
Spectral BCCB (ours), BCCB (spatial kernel), Conv2D ($k{=}3$), Spectral Circulant (1D on flattened input), Circulant (1D, spatial), and Dense.
For transfer (frozen features), a fixed CIFAR-10 encoder (ResNet) produces pooled features of dimension $2048$ which feed a Bayesian head:
\[
\underbrace{\RFFT\text{Circulant1D}(2048,\,K)}_{\text{spectral head}} 
\;\to\; \tanh 
 \;\to\; \text{Linear}(2048,\,10).
\]
We vary the active half-spectrum size $K\in\{d/2{+}1,\,768,\,512,\,256,\,128,\,64\}$ with $d{=}2048$.
For the ViT experiment (Sec .~\ref {sec:exp-vit}), we use a deliberately small ViT configuration; our goal is not to compete with state-of-the-art CIFAR-10 models but to test whether spectral circulant projections behave reasonably under a standard transformer training recipe. In the spectral variant, we replace all square linear projections (e.g.\ $d_{\text{model}}\!\to\!d_{\text{model}}$ in $Q/K/V$ and output projections) by SpectralCirculant1d layers, keeping all other architectural hyperparameters fixed.

\textbf{Training.} Bayesian models (MNIST/FMNIST and frozen-feature heads) are trained with SVI (reparameterized Gaussian) using Adam ($\eta{=}10^{-2}$), three seeds, and unitary FFTs. From scratch (MNIST/FMNIST) we optimize for $1000$ steps. Transfer heads use the same optimizer and schedule across all head architectures. All spectral SVI models use the same Hermitian-aware low-rank guide with rank $r{=}8$. For the ViT experiment, we train both the dense and spectral variants with the same deterministic recipe: Adan (ADAptive Nesterov momentum; \citealt{xie2022adan}) with cosine decay and 5-epoch warmup, peak learning rate $10^{-3}$, weight decay $10^{-4}$, batch size $128$, and up to $200$ epochs. We select the best checkpoint for each model by validation loss and report test metrics at that checkpoint.

\textbf{Metrics.} ID: Accuracy, NLL, Brier, ECE (10 bins), MCE. OOD: AUROC and FPR@95\%TPR from predictive entropy. For the transfer experiment, we report per-severity summaries and highlight severity level 5 (S5). For the ViT comparison, we report test accuracy, NLL, and parameter counts for the projection layers.

\subsection{Parameter Counts}
\label{app:param-counts}

Let $D{=}H\!\times\!W$ with $H{=}W{=}28$ ($D{=}784$). Weights exclude biases; biases shown in parentheses.
\begin{table}[H]
\centering
\small
\caption{Parameter counts on MNIST ($D{=}784$).}
\label{tab:param_counts_mnist}
\begin{threeparttable}
\begin{tabular}{@{}l r r r@{}}
\toprule
Model & Weights & Biases & Total \\
\midrule
SpectralCirculant & 8{,}624 & 11 & 8{,}635 \\
SpectralBCCB ($C_{\text{in}}{=}C_{\text{out}}{=}1$) & 8{,}624 & 11 & 8{,}635 \\
BCCB (spatial) & 8{,}624 & 11 & 8{,}635 \\
Conv2D ($C_{\text{out}}{=}8$, $k{=}3$) & 62{,}792 & 18 & 62{,}810 \\
Dense ($D\!\to\!D\!\to\!10$) & 622{,}496 & 794 & 623{,}290 \\
\bottomrule
\end{tabular}
\begin{tablenotes}[para,flushleft]
\footnotesize
Conv2D’s large flattened head ($8HW\!\to\!10$) dominates its parameters; Dense is $\sim$72$\times$ larger than spectral models.
\end{tablenotes}
\end{threeparttable}
\end{table}

\subsection{Additional MNIST $\rightarrow$ Fashion-MNIST Results}
\label{app:mnist-details}

Under matched budgets (1000 SVI steps), spectral models separate ID/OOD entropies better than Conv2D at much smaller capacity; Dense remains strongest but is not parameter-matched.

\begin{table*}[t]
\centering
\scriptsize
\setlength{\tabcolsep}{3pt}
\caption{MNIST (ID) \& Fashion-MNIST (OOD). Mean $\pm$ std over 3 seeds. OOD metrics from predictive entropy.}
\label{tab:mnist_main}
\begin{tabular}{lcccccc}
\toprule
Model & Acc. & Brier & ECE & MCE & AUROC & FPR@95 \\
\midrule
SpectralBCCB      & $0.919 \pm 0.003$ & $0.124 \pm 0.005$ & $0.016 \pm 0.005$ & $0.179 \pm 0.081$ & $0.8112 \pm 0.0245$ & $0.6156 \pm 0.0328$ \\
Conv2D            & $0.962 \pm 0.002$ & $0.076 \pm 0.005$ & $0.038 \pm 0.003$ & $0.573 \pm 0.069$ & $0.6205 \pm 0.0545$ & $0.7220 \pm 0.1163$ \\
BCCB              & $0.899 \pm 0.003$ & $0.182 \pm 0.007$ & $0.083 \pm 0.004$ & $0.373 \pm 0.008$ & $0.7296 \pm 0.0347$ & $0.8436 \pm 0.0178$ \\
SpectralCirculant & $0.921 \pm 0.005$ & $0.120 \pm 0.005$ & $0.022 \pm 0.003$ & $0.195 \pm 0.060$ & $0.8293 \pm 0.0242$ & $0.6369 \pm 0.0296$ \\
Circulant         & $0.903 \pm 0.006$ & $0.175 \pm 0.012$ & $0.081 \pm 0.005$ & $0.426 \pm 0.036$ & $0.7811 \pm 0.0190$ & $0.8128 \pm 0.0091$ \\
Dense             & $0.974 \pm 0.001$ & $0.048 \pm 0.001$ & $0.024 \pm 0.001$ & $0.623 \pm 0.028$ & $0.8397 \pm 0.0102$ & $0.5023 \pm 0.0395$ \\
\bottomrule
\end{tabular}
\end{table*}

\begin{figure}[H]
  \centering

  \begin{subfigure}{0.48\linewidth}
    \centering
    \includegraphics[width=\linewidth]{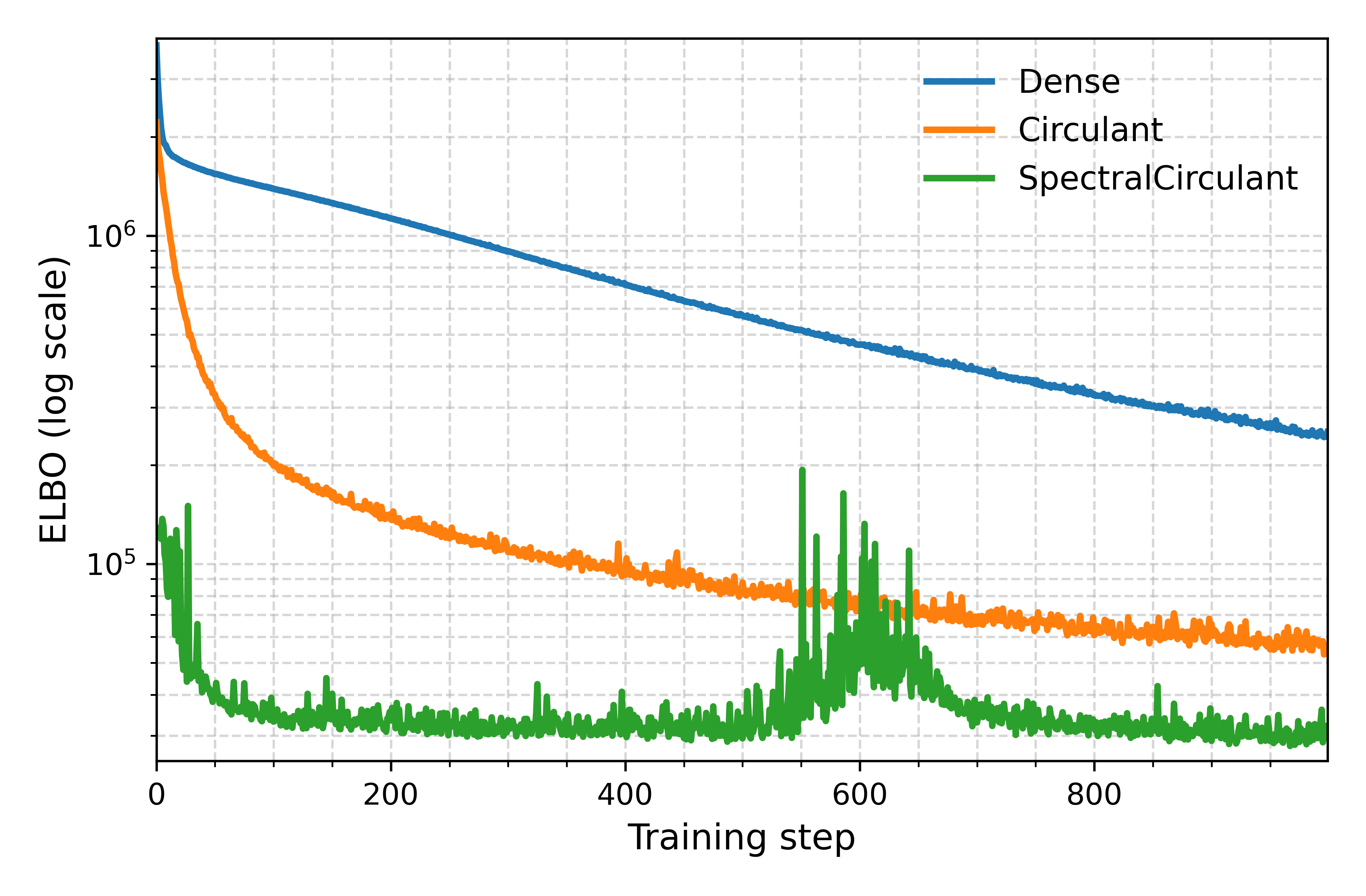}
    \caption{MNIST, 1D models.}
    \label{fig:mnist_elbo_1d}
  \end{subfigure}
  \hfill
  \begin{subfigure}{0.48\linewidth}
    \centering
    \includegraphics[width=\linewidth]{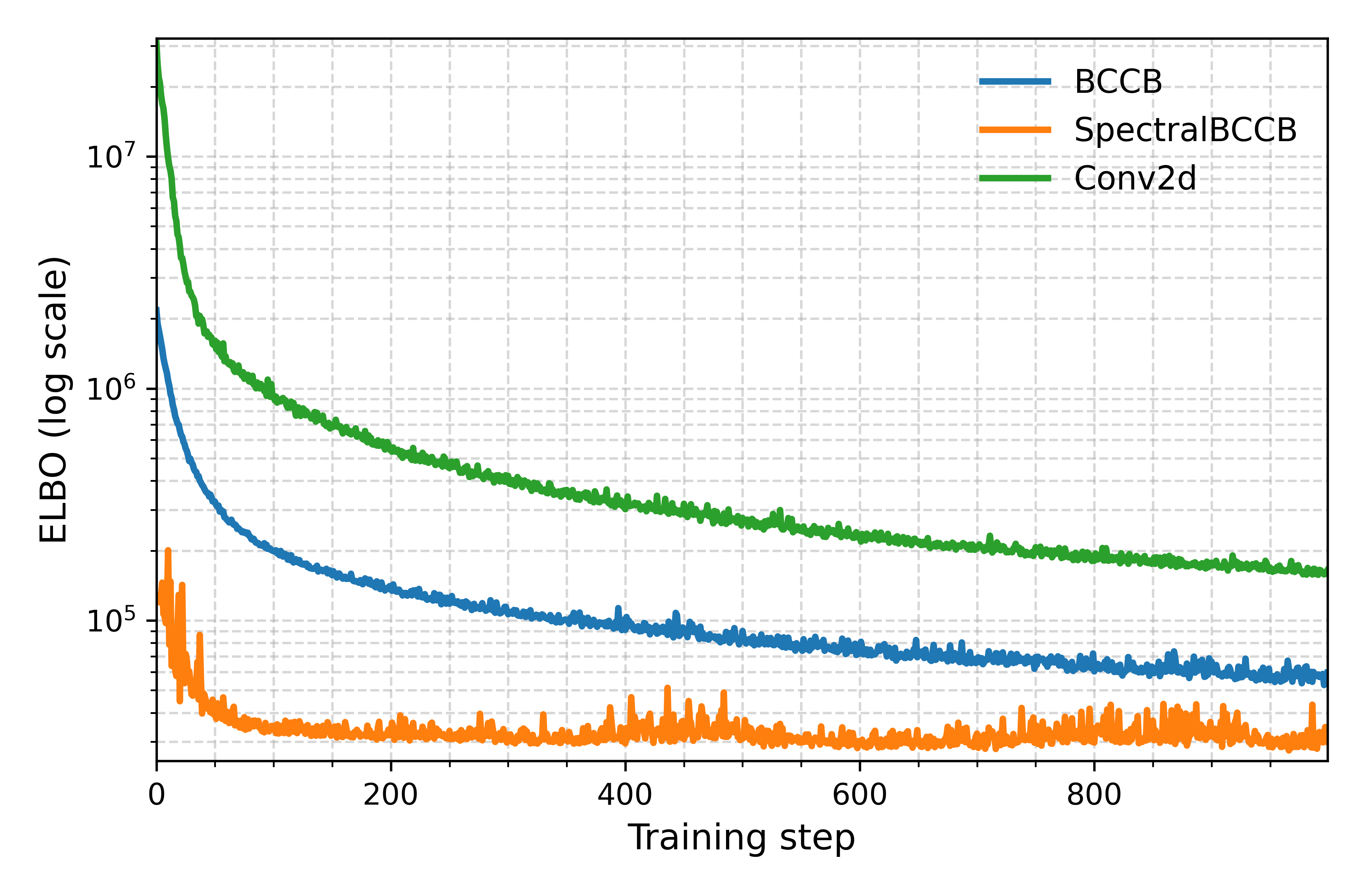}
    \caption{MNIST, 2D models.}
    \label{fig:mnist_elbo_2d}
  \end{subfigure}

  \caption{Training dynamics with log-scaled $y$-axis. 
  Negative ELBO / loss for MNIST (1D and 2D models). 
  SpectralBCCB stabilizes earlier in all settings.}
  \label{fig:training_dynamics}
\end{figure}

\subsection{Deep-kernel Heads: Spectrum Ablation and Geometry}
\label{app:dk-ablation}

We freeze a pretrained CIFAR-10 encoder (pooled features of dimension $2048$) and train SVI heads: Dense ($h\in\{32,128\}$), Circulant (spatial), and \textbf{Spectral Circulant (RFFT)} with active half-spectrum size $K$.

\paragraph{Summary.}
Dense heads set a strong reference on NLL/ECE and OOD AUROC. Spectral heads \emph{track ID accuracy closely at substantially lower parameter counts} and \emph{degrade gracefully} as $K$ shrinks.

\begin{table}[t]
\centering
\scriptsize
\setlength{\tabcolsep}{2pt} 
\caption{Frozen CIFAR-10 features $\rightarrow$ SVI heads on CIFAR-10 (ID) and CIFAR-10-C (OOD). AUROC/FPR are averaged over severities S1--S5; severity 5 (S5) is shown separately.}
\label{tab:frozen_head}
\begin{tabular}{lcccccc}
\toprule
Model & Acc & NLL & ECE & mean AUROC & AUROC S5 & FPR@95 S5 \\
\midrule
Dense-32 (SVI)     & 0.971 & \textbf{0.098} & \textbf{0.009} & 0.6986 & 0.7819 & \textbf{0.5727} \\
Dense-128 (SVI)    & 0.971 & 0.118         & 0.012          & \textbf{0.7076} & \textbf{0.7946} & 0.5648 \\
RFFT (SVI) -- Full & 0.971 & 0.136         & 0.017          & 0.6816 & 0.7165 & 0.5966 \\
Circulant (SVI)    & 0.965 & 0.433         & 0.026          & 0.6400 & 0.7182 & 0.6323 \\
\bottomrule
\end{tabular}%
\end{table}

\paragraph{Entropy curves at severity 5.}
\begin{figure}[H]
  \centering
  \includegraphics[width=\linewidth]{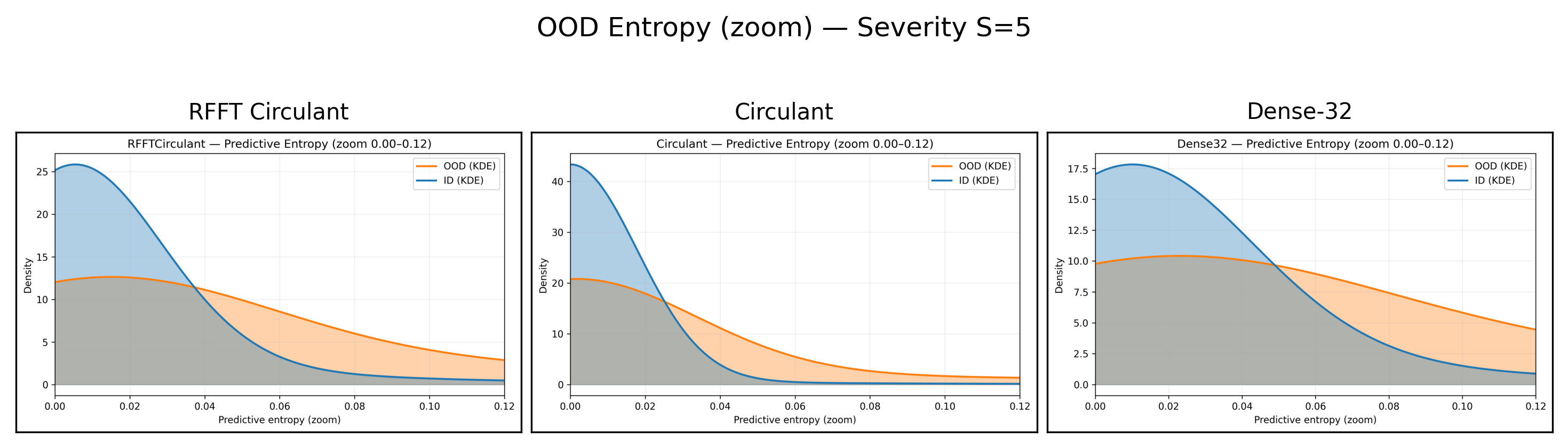}
  \caption{Predictive-entropy KDEs on CIFAR-10 (ID, blue) vs CIFAR-10-C (OOD, severity~5, orange) for frozen-feature heads.}
  \label{fig:ood_entropy_s5}
\end{figure}

\paragraph{Active spectrum ablation with parameter counts.}
For $d{=}2048$, the spectral head’s weight degrees of freedom are
\[
\#\text{spec-weights}(K)=
\begin{cases}
2K-2,& K=d/2{+}1,\\
2K-1,& \text{otherwise},
\end{cases}
\]
and the classifier head adds $2048\times 10=20{,}480$ weights. We report weights \emph{excluding} biases.

\begin{table}[t]
\centering
\scriptsize
\setlength{\tabcolsep}{2pt}
\caption{RFFT head ablation on frozen CIFAR-10 features (CIFAR-10 ID, CIFAR-10-C OOD). Parameter counts include spectral head $+$ linear classifier (weights only). Compression is relative to Dense-32 (65{,}856 weights).}
\label{tab:k_ablation_params}
\begin{tabular}{lccccc}
\toprule
$K$ & Weights & Comp.\ vs Dense-32 ($\times$) & Acc & mean AUROC & FPR@95 (S5) \\
\midrule
$d/2{+}1=1025$ & 22,528 & 2.92 & 0.971 & 0.6816 & 0.5966 \\
768            & 22,015 & 2.99 & 0.967 & 0.6846 & 0.6102 \\
512            & 21,503 & 3.06 & 0.967 & 0.6710 & 0.6264 \\
256            & 20,991 & 3.14 & 0.967 & 0.6664 & 0.6300 \\
128            & 20,735 & 3.18 & 0.965 & 0.6650 & 0.6447 \\
64             & 20,607 & 3.20 & 0.964 & 0.6615 & 0.6512 \\
\bottomrule
\end{tabular}%
\end{table}

\paragraph{Lipschitz geometry and robustness.}
For Dense-32 and the full RFFT head, we also inspect the distributions of
$\widehat{\operatorname{Lip}}(f_{\btheta})$, posterior margins $m_{\btheta}(\vx,y)$, and
certified radii $r_{\mathrm{cert}}(\vx,y;\btheta)$ on CIFAR-10 test features.
The spectral head yields a substantially smaller global Lipschitz bound
($\approx$200 vs.\ $\approx$5200) and \emph{larger} certified radii in feature space,
reflecting a smoother dependence on the representation.

\begin{figure}[H]
  \centering
  \begin{subfigure}{0.48\linewidth}
    \centering
    \includegraphics[width=\linewidth]{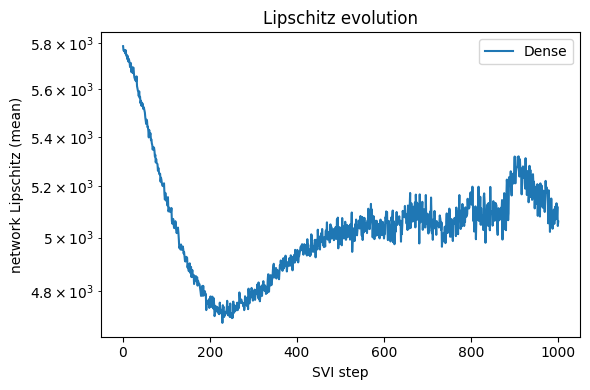}
    \caption{Dense head}
    \label{fig:lip_evolution_dense}
  \end{subfigure}\hfill
  \begin{subfigure}{0.48\linewidth}
    \centering
    \includegraphics[width=\linewidth]{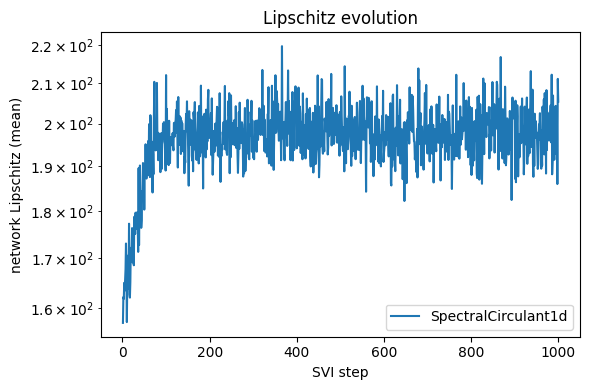}
    \caption{SpectralCirculant1d head}
    \label{fig:lip_evolution_rfft}
  \end{subfigure}
  \caption{
    Training dynamics of the global Lipschitz upper bound
    $\widehat{\operatorname{Lip}}(f_{\btheta})=\prod_{\ell=1}^{L}\|\bW_\ell(\btheta)\|_2$
    for Bayesian heads on frozen CIFAR-10 features.
    The SpectralCirculant1d head stays in a substantially lower Lipschitz regime
    throughout SVI compared to the Dense head.
  }
  \label{fig:lip_evolution_dk}
\end{figure}

\begin{figure}[H]
  \centering
  \includegraphics[width=0.6\linewidth]{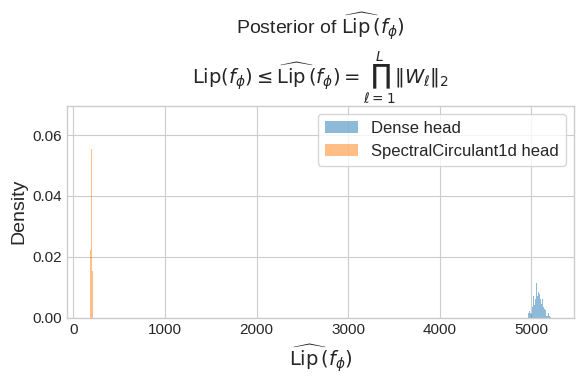}
  \caption{
    Posterior distributions of the global Lipschitz upper bound
    $\widehat{\operatorname{Lip}}(f_{\btheta})$ for Dense vs SpectralCirculant1d
    heads in the deep-kernel setting (frozen CIFAR-10 encoder).
    Despite similar predictive performance, the spectral head yields much
    smaller Lipschitz constants, reflecting a smoother dependence on the
    frozen representation.
  }
  \label{fig:lip_posterior_dk}
\end{figure}
\begin{figure}[t]
  \centering
  \includegraphics[width=0.5\linewidth]{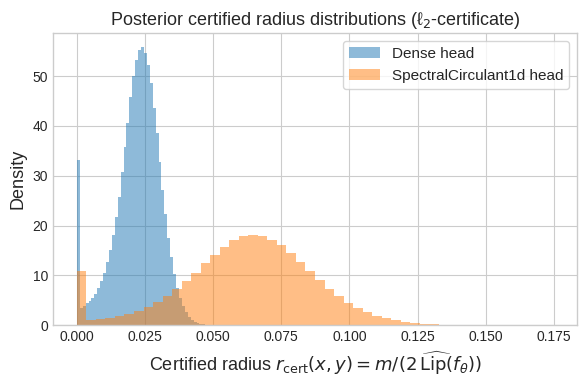}
  \caption{Frozen CIFAR-10 features: spectral-circulant SVI heads yield substantially larger certified $\ell_2$ radii due to smaller global Lipschitz bounds $\widehat{\operatorname{Lip}}(f_{\btheta})$.}
  \label{fig:dk-geometry-main}
\end{figure}

\section{Vision Transformer Experiments}
\label{app:vit-details}

We next ask whether the same spectral circulant primitives are useful in a modern transformer architecture trained deterministically. We consider a compact ViT encoder with patch embedding, multi-head self-attention, and an MLP block per layer. The \emph{dense} baseline uses standard linear projections for $Q$, $K$, $V$, the attention output, and all MLP layers. The \emph{spectral} variant keeps the architecture and training protocol fixed, and replaces all \emph{square} linear projections ($d_{\text{model}}\!\to\!d_{\text{model}}$) by \texttt{SpectralCirculant1d} layers; non-square projections (e.g.\ patch embedding and classifier head) remain dense.

Both models are trained with the same optimizer and schedule described in Sec.~\ref{sec:experiments-main}: Adan with warmup--cosine decay, peak learning rate $10^{-3}$, weight decay $10^{-4}$, batch size $128$, and up to $200$ epochs. For each configuration we select the checkpoint with the lowest validation loss.

\paragraph{Training dynamics and performance on CIFAR-10.}
On CIFAR-10 we use $32{\times}32$ images, patch size $4$, and a ViT with $d_{\text{model}}{=}128$ as described above. Under the shared training recipe, both the dense and spectral models converge stably. The spectral variant attains its best validation loss later in training (epoch 30 vs.\ 16 for the dense model), but we do not observe optimization instabilities.

On the CIFAR-10 test set, the best checkpoints achieve
\[
\begin{aligned}
\text{Dense ViT (CIFAR-10):}\quad
  & \text{acc}=0.6851,\;\text{NLL}=0.9660,\\
\text{Spectral ViT (CIFAR-10):}\quad
  & \text{acc}=0.7157,\;\text{NLL}=0.8782.
\end{aligned}
\]
Under identical hyperparameters, the spectral circulant ViT therefore attains higher accuracy and lower negative log-likelihood than the dense baseline.

\paragraph{Tiny ImageNet.}
To test whether this behaviour is specific to CIFAR-10, we repeat the experiment on Tiny ImageNet (200 classes, $64{\times}64$ images), using patch size $8$ so that the number of tokens per image matches the CIFAR-10 setup. We keep the architecture and training hyperparameters fixed and only change the number of classes. Under this matched recipe, both models again converge reliably, and the spectral ViT prefers a substantially later best-validation epoch: epoch~13 for the dense model vs.\ epoch~58 for the spectral model.

On Tiny ImageNet, the dense ViT achieves $30.7\%$ top-1 accuracy and an NLL of $3.06$ with $1.25$M parameters, while the spectral ViT attains $34.4\%$ top-1 accuracy and an NLL of $2.94$ with $0.85$M parameters. While these absolute numbers are below heavily tuned CNN and transformer baselines reported in the literature, they are consistent with small ViT models trained from scratch on Tiny ImageNet without pretraining or aggressive regularization
\citep[e.g.][]{ma2017tinyimagenet,cschoeller_vit_tinyimagenet}. Across both datasets the spectral variant selects a substantially later best-validation epoch than the dense model, which is consistent with the hypothesis that the spectral parametrization delays overfitting and allows the model to continue improving for longer under the same schedule. The qualitative trend mirrors our CIFAR-10 results: spectral circulant projections yield better likelihood and higher accuracy at a reduced parameter budget under a shared, untuned training recipe.

\paragraph{Parameter counts.}
Replacing each $d_{\text{model}}\!\times\!d_{\text{model}}$ linear map by a circulant matrix reduces the number of parameters in those projections from $O(d_{\text{model}}^2)$ to $O(d_{\text{model}})$. In our CIFAR-10 configuration the total parameter counts are
\[
\begin{aligned}
& \text{Dense ViT (CIFAR-10): } 1{,}205{,}898 \text{ parameters} \\
& \text{Spectral ViT (CIFAR-10): } 811{,}218 \text{ parameters},
\end{aligned}
\]
while in the Tiny ImageNet setting (with a larger classifier head) we obtain
\[
\begin{aligned}
& \text{Dense ViT (Tiny): } 1{,}248{,}840 \text{ parameters} \\
& \text{Spectral ViT (Tiny): } 854{,}160 \text{ parameters}.
\end{aligned}
\]
In both cases the spectral variant uses roughly $1.4$--$1.5\times$ fewer parameters overall. Table~\ref{tab:vit-main} summarizes performance and model size across datasets.
\begin{figure}[t]
  \centering
  \includegraphics[width=0.5\linewidth]{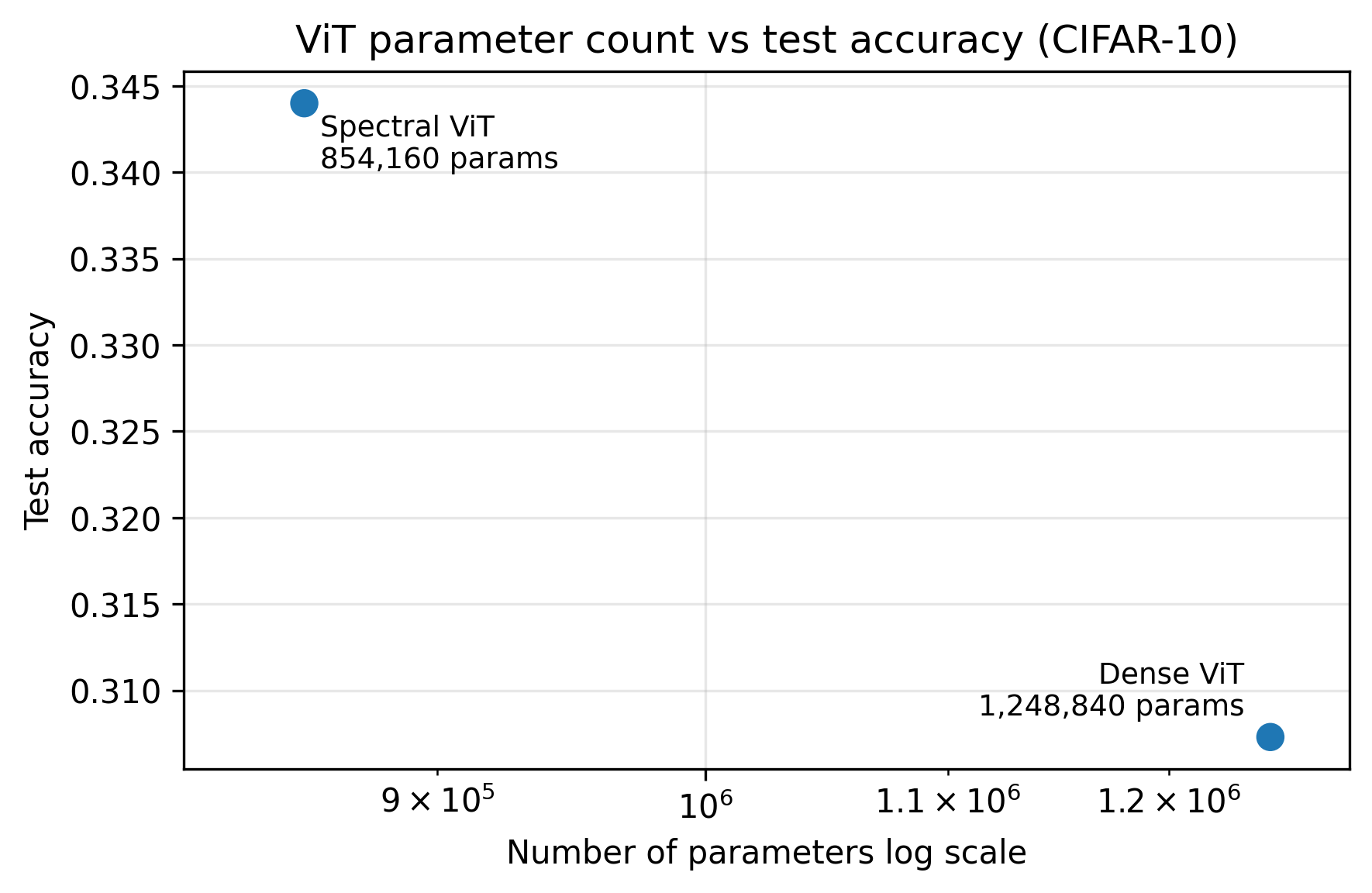}
  \caption{Tiny ImageNet test accuracy vs parameter count.}
  \label{fig:vit-acc-params}
\end{figure}

Overall, these experiments show that spectral circulant layers are not restricted to small Bayesian networks or probabilistic heads: they can be used as drop-in replacements for dense projections in ViT-like architectures, trained with standard first-order optimizers, and provide a favourable accuracy--parameter trade-off without architecture-specific hyperparameter tuning.

\paragraph{Curvature diagnostics.}
\label{app:curvature-details}
To better understand how the spectral parametrization affects optimization geometry in the attention projections, we estimate dominant curvature directions using Hessian--vector products and Gauss--Newton vector products, following standard practice for curvature diagnostics in deep networks \citep{ghorbani2019hessian}. We report the resulting dominant eigenvalues both globally and in an attention-only subspace (restricted to the $Q/K/V$ and output projection parameters) at the final selected checkpoints.
Using Hessian--vector products and a small number of power-iteration steps, following standard practice for curvature diagnostics in deep networks \citep{ghorbani2019hessian}, we approximate the dominant eigenvalues of the Hessian and Gauss--Newton matrices of the validation loss on a small held-out batch at the final checkpoints. Table~\ref{tab:vit_eigs} reports results for both CIFAR-10 and Tiny ImageNet, considering the full parameter space and the same attention-only subspace described above.

On CIFAR-10, in the attention subspace the spectral ViT has a much smaller dominant (largest-magnitude) Hessian eigenvalue than the dense model (approximately $8.5\times$ smaller), and an even more pronounced reduction in the top Gauss--Newton eigenvalue (about $21\times$). This indicates substantially flatter curvature around the optimum in the directions where \texttt{SpectralCirculant1d} layers are applied. When all parameters are included, the spectral ViT exhibits moderately larger dominant eigenvalues (roughly $1.6\times$ for both Hessian and Gauss--Newton), suggesting that sharp directions are less concentrated in attention and more concentrated in the remaining dense layers (MLPs and classifier head). The same qualitative pattern appears on Tiny ImageNet: spectral projections yield significantly smaller dominant eigenvalues in the attention subspace (about $25\times$ smaller for the Hessian and $16\times$ for Gauss--Newton), while the global top eigenvalues are larger for the spectral model, reflecting stronger curvature elsewhere in the network.

These diagnostics are local and parameterization-dependent, and we use them only for relative comparison between the two ViT variants. Taken together with the improved test accuracy and reduced parameter counts in Table~\ref{tab:vit-main}, they support the view that spectral circulant projections offer a favourable compactness--geometry trade-off in transformer architectures.
\begin{table}[t]
  \centering
  \caption{ViT curvature diagnostics at the selected final checkpoints.
    $\lambda_{\mathrm{dom}}(H)$ denotes the dominant Hessian eigenvalue in magnitude (reported with its sign; negative values indicate a direction of negative curvature), and $\lambda_{\max}(\mathrm{GN})$ is the largest Gauss--Newton eigenvalue (nonnegative by construction). We report results for the full parameter space (Global) and an attention-only subspace restricted to the $Q/K/V$ and output projection parameters.}
    
  \label{tab:vit_eigs}
  \begin{adjustbox}{max width=\columnwidth}
    \scriptsize
    \setlength{\tabcolsep}{2pt}
    \begin{tabular}{llcccc}
    \toprule
    Dataset & Subspace &
    \multicolumn{2}{c}{Dense ViT} &
    \multicolumn{2}{c}{Spectral ViT} \\
    & &
    $\lambda_{\mathrm{dom}}(H)$ & $\lambda_{\max}(\text{GN})$ &
    $\lambda_{\mathrm{dom}}(H)$ & $\lambda_{\max}(\text{GN})$ \\
    \midrule
    CIFAR-10      & Global         & 12.3 & 1621  & 19.6   & 2511  \\
                  & Attention-only &  6.3 &  858  &  0.75  &   40.9 \\
    \midrule
    Tiny ImageNet & Global         & 21.9 & 2270  & $-51.4$ & 9204  \\
                  & Attention-only &  9.3 & 1132  &  0.38  &   70.1 \\
    \bottomrule
    \end{tabular}
  \end{adjustbox}
\end{table}

\end{document}